\documentclass[11pt]{colt2018} 

\title[Learning Graph]{On Learning Graphs with Edge-Detecting Queries}
\usepackage{times}

\usepackage{graphicx}
\graphicspath{{images/}}
\makeatletter
\def\temp{dvips.def}
\ifx\Gin@driver\temp
\def\Ginclude@graphics#1{\def\temp{#1}---image \expandafter\strip@prefix\meaning\temp---}
\fi
\makeatother

\usepackage{amsmath}
\usepackage{amssymb}
\usepackage{color}

\newcommand{\ignore}[1]{}

\newcommand{\poly}{{\rm poly}}

\newcommand{\NO}{{\rm \mbox{``NO''}}}

\newcommand{\E}{{\bf E}}

\ignore{
\newtheorem{theorem}{Theorem}

\newtheorem{lem}[theorem]{lem}

\newtheorem{corollary}[theorem]{Corollary}

}
\newtheorem{fact}{Fact}

\usepackage{thmtools}
\usepackage{thm-restate}

\usepackage{hyperref}

\usepackage{cleveref}

\declaretheorem[name=Lemma]{lem}

\coltauthor{\Name{Hasan Abasi} \Email{hassan@cs.technion.ac.il}\\
    \Name{Nader H. Bshouty} \Email{bshouty@cs.technion.ac.il}\\
    \addr Technion, Haifa, Israel}

\begin{document}

\maketitle

\begin{abstract}
We consider the problem of learning a general graph $G=(V,E)$ using edge-detecting queries, where the number of vertices $|V|=n$ is given to the learner. The information theoretic lower bound gives $m\log n$ for the number of queries, where $m=|E|$ is the number of edges. 
In case the number of edges $m$ is also given to the learner, Angluin-Chen's Las Vegas algorithm \cite{AC08} runs in $4$ rounds and detects the edges in $O(m\log n)$ queries.
In the other harder case where the number of edges $m$ is unknown, their algorithm runs in $5$ rounds and asks $O(m\log n+\sqrt{m}\log^2 n)$ queries.
There have been two open problems: \emph{(i)} can the number of queries be reduced to $O(m\log n)$ in the second case, and, \emph{(ii)} can the number of rounds be reduced without substantially increasing the number of queries (in both cases).

For the first open problem (when $m$ is unknown) we give two algorithms. The first is an $O(1)$-round Las Vegas algorithm that asks $m\log n+\sqrt{m}(\log^{[k]}n)\log n$ queries for any constant $k$ where $\log^{[k]}n=\log \stackrel{k}{\cdots} \log n$. The second is an $O(\log^*n)$-round Las Vegas algorithm that asks $O(m\log n)$ queries. This solves the first open problem for any practical $n$, for example, $n<2^{65536}$. We also show that no deterministic algorithm can solve this problem in a constant number of rounds.

To solve the second problem we study the case when $m$ is known. We first show that any non-adaptive Monte Carlo algorithm (one-round) must ask at least $\Omega(m^2\log n)$ queries, and any two-round Las Vegas algorithm must ask at least $m^{4/3-o(1)}\log n$ queries on average. We then give two two-round Monte Carlo algorithms, the first asks $O(m^{4/3}\log n)$ queries for any $n$ and $m$, and the second asks $O(m\log n)$ queries when $n>2^m$. Finally, we give a $3$-round Monte Carlo algorithm that asks $O(m\log n)$ queries for any $n$ and $m$.
\end{abstract}

\begin{keywords}
Graphs Learning, Group Testing, Edge-Detecting Queries, Monte Carlo Algorithm, Las Vegas Algorithm.
\end{keywords}

\section{Introduction}\label{Int}
We consider the problem of learning a general graph $G=(V,E)$ using edge-detecting queries.
In edge-detecting queries the learning algorithm asks whether a subset of vertices $Q\subseteq V$ contains an edge in the graph $G$. That is, whether there are $u,v\in Q$ such that $\{u,v\}\in E$. The learner knows the number of vertices $|V|=n$. 

Graph learning is a well-studied problem. It has been studied for general graphs, \cite{AC08}  and also for specific graph families (i.e. matching, stars, cliques and other) see \cite{AA05,ABKRS04}.
This problem has also been generalized to learning a hyper-graph (where each edge consists of two or more vertices) see \cite{ABM15, ABM14, AC06}.
The motivation behind studying some graph families relevant to the problem above, was its various applications in different areas such molecular biology, chemistry and networks \cite{BGK05,AAR10}.
For example, the general graph case is motivate by problems from biology and chemistry where, given a set of molecules (chemicals), we need pairs that react with each other. In this case,
the vertices correspond to the molecules (chemicals), the edges to the reactions, and the queries to
experiments of putting a set of molecules (chemicals) together in a test tube and determining whether a reaction occurs. When multiple molecules (chemicals) are combined in one test tube, a reaction is detectable if and only if at least one pair of the molecules (chemicals) in the tube react. The task is to identify which pairs react using as few experiments as possible.
One more example of a problem encountered by molecular biologists, is the problem of finding a hidden match, when applying multiplex PCR in order to close the gaps left in a DNA strand after shotgun sequencing. See \cite{AA05,ABKRS04} and there references for more details.

For a general graph, the information theoretic lower bound gives $m\log n$ for the number of queries where $m=|E|$ is the number of edges. 
In case the number of edges $m$ is also given to the learner, Angluin-Chen's Las Vegas algorithm \cite{AC08} runs in $4$ rounds and detects the edges in $O(m\log n)$ queries.
In the other harder case where the number of edges $m$ is unknown, their algorithm runs in $5$ rounds and asks $O(m\log n+\sqrt{m}\log^2 n)$ queries.
There have been two open problems: \emph{(i)} can the number of queries be reduced to $O(m\log n)$ in the second case, and, \emph{(ii)} can the number of rounds be reduced without substantially increasing the number of queries (in both cases).

For the first open problem (when $m$ in unknown) we give two algorithms. The first is an $O(1)$-round algorithm that asks $m\log n+\sqrt{m}(\log^{[k]}n)\log n$ queries for any constant $k$. Here $\log^{[k]}n=\log \stackrel{k}{\cdots} \log n$. The second is an $O(\log^*n)$-round algorithm that asks $O(m\log n)$ queries. This solves the first open problem for any practical $n$, for example $n<2^{65536}$. We also show that no deterministic algorithm can solve this problem in a constant number of rounds.

To solve the second problem we study the problem when $m$ is known. We first show that any non-adaptive (one-round) Monte Carlo algorithm must ask at least $\Omega(m^2\log n)$ queries. We then give two two-round Monte Carlo algorithms, the first asks $O(m^{4/3}\log n)$ queries for any $n$ and $m$, and the second asks $O(m\log n)$ queries when $n>2^m$. Then we give a $3$-round Monte Carlo algorithm that asks $O(m\log n)$ queries for any $n$ and $m$. We also show that any two-round Las Vegas algorithm must ask at least $m^{4/3}$ queries on average and any deterministic two-round algorithm must ask at least $\Omega(m^2\log n)$ queries. Finally, we give a four-round deterministic algorithm that asks $m^{2+\epsilon}\log n$ queries for any constant $\epsilon$. The question whether there is an $O(1)$-round deterministic algorithm that asks $O(m\log n)$ queries remains open.

Our results are summarized in the following two tables
\subsection{Results For Unknown \texorpdfstring{$m$}{}}

\centerline{
\begin{tabular}{l|c|c|c}
\hline
& Lower Bound & Upper Bound &Poly. Time\\
\hline \hline
LV\&MC Randomized & & &\\
$O(1)$-Rounds&  &$m\log n$ &$m\log n+$ \\
&$m\log n$  &$\sqrt{m}(\log^{[k]}n)\log n$ &$\sqrt{m}(\log^{[k]}n)\log n$ \\
\hline \hline
LV\&MC Randomized & & &\\
$O(\log^* n)$-Round&$m\log n$  &$m\log n$ &$m\log n$ \\
\hline
\hline
\end{tabular}}
\newpage

\subsection{Results For Known \texorpdfstring{$m$}{}}
\centerline{
\begin{tabular}{l|c|c|c}
\hline
& Lower Bound & Upper Bound &Poly. Time\\
\hline \hline
{\bf Non-Adaptive} &Thm.\ref{LBNAMC} &$\Rightarrow$ &Thm.~\ref{UBNAMC}\\ \cline{2-4}
MC Randomized&${m^2}\log n$ &$m^2\log n$ &$m^2\log n$ \\
\hline
 {\bf Non-Adaptive} &\cite{DR82} &$\Rightarrow$ &\cite{B15}\\ \cline{2-4}
Deter. \& LV Rand.&$\frac{m^3}{\log m}\log n$ &${m^3}\log n$ &${m^3}\log n$\\
\hline \hline
{\bf Two Rounds} &&$\Rightarrow$ &Thm.\ref{MCTR}\&\ref{MCTRN}\\ \cline{3-4}
MC Randomized&OPEN &$m^{4/3}\log n$ &$m^{4/3}\log n$ \\
\cline{2-4}
\ \ \ \ \ \ \ \ \ $n\to\infty$&$m\log n$ &$m\log n$ &$m\log n$ \\
\hline
{\bf Two Rounds} & &$\Rightarrow$&Thm.~\ref{UBNAMC}\\ \cline{3-4}
LV Randomized&OPEN &$m^2\log n$ &$m^2\log n$ \\
\hline
{\bf Two Rounds} &Thm.\ref{UBDATR} &Thm.\ref{UBDTR} &\\ \cline{2-3}
Deterministic&{$\frac{m^2}{\log m}\log n$} &{$m^2\log n$}&OPEN\\
\hline \hline
{\bf Three Rounds} &I.T. &$\Rightarrow$ &Thm. \ref{MCTR2}\\ \cline{2-4}
MC Randomized&{${m}\log n$} &{${m}\log n$} & {${m}\log n$}\\
\hline
{\bf Three Rounds} &I.T. &$\Rightarrow$ &Thm.\ref{MCTR}\&\ref{LVTRN}\\  \cline{2-4}
LV Randomized&$m\log n$ &$m^{4/3}\log n$ &$m^{4/3}\log n$ \\ \cline{2-4}
\ \ \ \ \ \ \ \ \ $n\to\infty$&$m\log n$ &$m\log n$ &$m\log n$ \\
\hline
{\bf Three Rounds} & & &\\
Deterministic&OPEN &OPEN&OPEN\\
\hline
\hline
{\bf Four Rounds} &I.T. &$\Rightarrow$ &Thm. \ref{MCTR2}\\ \cline{3-4}
LV Randomized&{${m}\log n$} &{${m}\log n$} & {${m}\log n$}\\
\hline\hline
{\bf Five Rounds} & &$\Rightarrow$&Thm.\ref{DDDDD} \\ \cline{3-4}
Deterministic&OPEN &$m^{2+\epsilon}\log n$&$m^{2+\epsilon}\log n$\\
\hline \hline
\end{tabular}}
$$ $$
\section{Definitions and Preliminary Results}
Let $G=(V,E)$ be a simple (contains no loop or multiple edges) undirected graph with $|V|=n$ vertices and $|E|\le m$ edges. We call $G$ the {\it target $m$-graph}. The {\it learner} knows $n$. We study both cases where $m$ is known to the learner and when it is not. The learner can ask an {\it edge-detecting queries} (or just a {\it query}) $Q\subseteq V$ to an {\it oracle} ${\cal O}_G$. The answer to the query $Q$ is ${\cal O}_G(Q)=$``YES'' if there are $u,v\in Q$ such that $\{u,v\}\in E$ and ``NO'' otherwise.

We say that a non-simple graph $G=(V,E)$ is {\it $m$-Loop} if the graph contains at most $m$ loops. That is, the graph contains at most $m$ edges where each edge connects a vertex to itself. Learning $m$-Loop is equivalent to the problem of group testing~\cite{DH00,DH06}.

For a graph $G=(V,E)$ and a subset of vertices $V'\subseteq V$ we define the {\it neighbours of $V'$}, $\Gamma_G(V')$, as the set of all vertices $u$ in $V\backslash V'$ such that there is $v\in V'$ where $\{u,v\}\in E$. When $V'=\{v\}$ then we write $\Gamma_G(v)$ and call it the neighbours of $v$. We say that $V'$ is an independent set if there are no edges between the vertices in~$V'$. A set of vertices $I\subset V$ is called {\it independent set} in $G=(V,E)$ if $(I\times I)\cap E=\emptyset$.

We will denote by $[n]:=\{1,2,\ldots,n\}$ and $[m,n]=\{m,m+1,\ldots,n\}$.

\subsection{Preliminary Results for Randomized Algorithms}
The following lemma is from~\cite{AC08}:
\begin{restatable}{lem}{flem}\label{onequery} Let $G=(V,E)$ be the target graph with $n$ vertices and $m$ edges.
Let $Q\subseteq V$ be a random query where for each vertex $i\in V$, $i$ is included in $Q$ independently with probability~$p$. Let $\{u,v\}\not\in E$. Then
\begin{enumerate}
 \item If $|\Gamma_G(\{u,v\})|\le  r$, then with probability at least $p^2(1-rp-mp^2)$, $u,v\in Q$ and ${\cal O}_G(Q)=$``NO''.
 \item If $|\Gamma_G(\{u,v\})|\ge  r$, then with probability at most $p^2(1-p)^r$, $u,v\in Q$ and ${\cal O}_G(Q)=$``NO''.
\end{enumerate}
\end{restatable}

\begin{proof} Let $\Gamma_G(\{u,v\})=\{w_1,\ldots, w_{r'}\}$. Let $\{u_1,v_1\},\ldots,\{u_{m-r'},v_{m-r'}\}$ be all the edges of $G$ that none of their endpoints are $u$ or $v$. When $r'\le r$,
\begin{eqnarray*}
\Pr[u,v\in Q\ \wedge\ {\cal O}_G(Q)=\NO]&=& \Pr[u,v\in Q]\Pr[{\cal O}_G(Q)=\NO\ |\ u,v\in Q]\\
&=& p^2\Pr[(\forall i)w_i\not\in Q\ \wedge\ (\forall j)\{u_j,v_j\}\not\subseteq Q]\\
&=& p^2(1-\Pr[(\exists i)w_i\in Q\ \vee\  (\exists j)\{u_j,v_j\}\subseteq Q])\\
&\ge& p^2(1-r'p-(m-r')p^2)\\
&\ge& p^2(1-rp-mp^2).
\end{eqnarray*}
and when $r'\ge r$,
\begin{eqnarray*}
\Pr[u,v\in Q\ \wedge\ {\cal O}_G(Q)=\NO]&=&  p^2\Pr[(\forall i)w_i\not\in Q\ \wedge\ (\forall j)\{u_j,v_j\}\not\subset Q]\\
&\le& p^2\Pr[(\forall i)w_i\not\in Q]\\
&=& p^2(1-p)^{r'}\le p^2(1-p)^r.
\end{eqnarray*}
\end{proof}

\begin{figure}[t!]
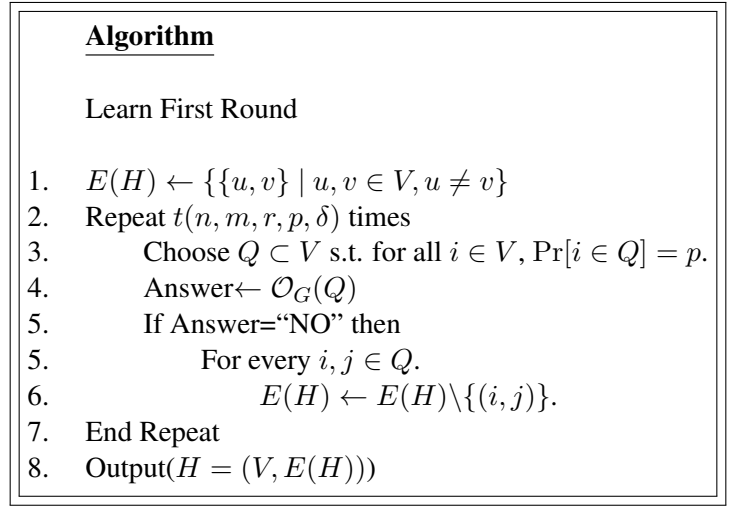

  \begin{center}
   \fbox{\fbox{\begin{minipage}{28em}
  \begin{tabbing}
  xxxx\=xxxx\=xxxx\=xxxx\= \kill
  \>\underline{{\bf Algorithm}}\\ \\
  \> Learn First Round\\ \\
  1. \>$E(H)\gets \{\{u,v\}\ |\ u,v\in V, u\not=v\}$\\
  2. \> Repeat $t(n,m,r,p,\delta)$ times\\
  3. \>\> Choose $Q\subset V$ s.t. for all $i\in V$, $\Pr[i\in Q]=p$.\\
  4. \>\> Answer$\gets{\cal O}_G(Q)$\\
  5. \>\> If Answer=``NO'' then\\
  5. \>\> \> For every $i,j\in Q$.\\
  6. \>\> \> \> $E(H)\gets E(H)\backslash \{(i,j)\}$.\\
  7. \> End Repeat\\
  8. \> Output($H=(V,E(H))$)
  \end{tabbing}\end{minipage}}}
  \end{center}
	\caption{First Round in the Algorithm.}
	\label{Alg1}
	\end{figure}

\begin{restatable}{lem}{slem}\label{main}
Let $G=(V,E)$ be the target graph with $n$ vertices and $m$ edges. Consider the algorithm in Figure~\ref{Alg1} with
$$t:=t(n,m,r,p,\delta)=\frac{2\ln n+\ln(1/\delta)}{p^2(1-rp-mp^2)}.$$ Then for $E_r=\{\{u,v\}:|\Gamma_G(\{u,v\})|>r\}$,
\begin{enumerate}
\item $E\subseteq E(H).$
\item With probability at least $1-\delta$, $E(H)\backslash (E\cup E_r)=\emptyset$, i.e., all the edges $\{u,v\}\in E(H)\backslash E$ satisfies $|\Gamma_G(\{u,v\})|>r$.
\item $\E\left[|E(H)\backslash (E\cup E_r)|\right]\le \delta$.
\end{enumerate}
\end{restatable}
\begin{proof} Consider the algorithm in Figure~\ref{Alg1}. An edge $\{i,j\}$ is removed from the graph $H$ if and only if $i,j\in Q$ and ${\cal O}_G(Q)=$``NO''. Therefore, if $\{i,j\}\in E$ then $\{i,j\}\in E(H)$. This proves {\it 1.}

We now prove {\it 2.} Let $\{u,v\}$ be such that $|\Gamma(\{u,v\})|\le r$. The probability that $\{u,v\}\in E(H)\backslash E$ is the probability that for all the queries $Q^{(i)}$, $i=1,\ldots,t$ in the algorithm, either $\{u,v\}\not\subseteq Q^{(i)}$ or ${\cal O}_G(Q^{(i)})=$``YES''. By Lemma~\ref{onequery}, this probability is at most $(1-p^2(1-rp-mp^2))^t$. Therefore, the probability that there is $\{u,v\}$ such that $|\Gamma(\{u,v\})|\le r$ and $\{u,v\}\in E(H)\backslash E$ is at most
$$n^2(1-p^2(1-rp-mp^2))^t\le \delta.$$

The expected number of edges in $E(H)\backslash (E\cup E_r)$ is also at most $n^2(1-p^2(1-rp-mp^2))^t\le \delta.$
\end{proof}

\section{Non-Adaptive Algorithms when \texorpdfstring{$m$}{} is Known}
In this section we study non-adaptive learning algorithms when $m$ is known to the algorithm. We first give a polynomial time Monte Carlo algorithm that asks $O(m^2\log n)$ queries and then prove the lower bound $\Omega(m^2\log n)$ for the number of queries. For the non-adaptive deterministic algorithm, there is a polynomial time algorithm that asks $O(m^3\log n)$ queries and it is known that the lower bound for the number of queries is $\Omega(m^3\log n/\log m)$~\cite{ABM15,DH06}. For Las Vegas algorithm, a non-adaptive algorithm must be deterministic because the success probability is $1$. Therefore, the bounds for the number of queries of deterministic algorithms apply also for Las Vegas algorithms.

\subsection{Upper Bound for Randomized Non-Adaptive Algorithm}
We first state the following theorem: 
\begin{restatable}{theorem}{fthm}\label{UBNAMC} There is a non-adaptive Monte Carlo randomized learning algorithm with $1/poly(n)$ error probability for $m$-Graph that asks $O(m^2\log n)$ queries.

There is a two-round Las Vegas randomized learning algorithm for $m$-Graph that asks $O(m^2\log n)$ expected number of queries.
\end{restatable}

\begin{proof} Consider the algorithm in Figure~\ref{Alg1} with $r=m$, $p=1/(2m)$ and $\delta=1/poly(n)$. Then by Lemma~\ref{main}, the number of queries is $t(n,m,r,p,\delta)=$ $O(m^2\log n)$. Again by Lemma~\ref{main}, $E\subseteq E(H)$ and with probability at least $1-\delta$ all the edges $\{u,v\}\in E(H)\backslash E$ satisfies $|\Gamma(\{u,v\})|>m+1$. Since $G$ has at most $m$ edges we must have $|\Gamma(\{u,v\})|\le m+1$ for any $u,v\in V$ and therefore with probability at least $1-\delta$, $E(H)= E$.

For the Las Vegas algorithm we add another round that asks a query for each edge in $E(H)$. The number of expected edges is less than $m+1$ and therefore the expected number of queries in the second round is less than $m+1$.
\end{proof}

\subsection{Lower Bound for Randomized Non-Adaptive Algorithm}

\begin{restatable}{lem}{tlem}\label{RNLL} Any non-adaptive Monte Carlo randomized learning algorithm with error probability at most $1/2$ for $m$-Loop must ask at least $\Omega(m\log n)$ queries.
\end{restatable}
\begin{proof} The number of $m$-loops with $n$ vertices is $\binom{n}{m}$. Therefore, by the information theoretic lower bound and Lemma~2 in~\cite{ABM14} the result follows.
\end{proof}
\begin{theorem}\label{LBNAMC} Any non-adaptive Monte Carlo randomized learning algorithm with error probability at most $1/4$ for $m$-Graph must ask at least $\Omega(m^2\log n)$ queries.
\end{theorem}
\begin{proof} Let ${\cal A}(s,{\cal O}_G)$ be a Monte Carlo randomized non-adaptive learning algorithm that learns $m$-Graph over the vertices $[n]$ with error probability at most $1/4$, where $s\in \{0,1\}^*$ is a random seed and ${\cal O}_G$ is the query oracle to the target graph $G$. Let $A_s$ be the set of queries that are asked when the random seed is $s$. Suppose, on the contrary, that for all $s$, $|A_s|\le cm^2\log n$, for some $c=o(1)$. Let $A_{s,i}$, $i=1,\ldots,m/2$ be the set of queries $Q$ in $A_s$ that contains the vertex $i$ and does not contain any of the vertices in $[m/2]\backslash \{i\}$. That is $Q\cap [m/2]=\{i\}$. Then, over the uniform distribution, $(m/2)\E_{i\in [m/2]}[|A_{s,i}|]\le |A_s|<cm^2\log n$. Therefore, $\E_{i\in [m/2]}[|A_{s,i}|]\le 2cm\log n$. Thus, by Markov bound, we get $\Pr_{i\in [m/2]}[|A_{s,i}|\ge 8cm\log n]< 1/4$ and $\Pr_{i\in [m/2]}[|A_{s,i}|< 8cm\log n]> 3/4$.

Now for any $i\in [m/2]$ and $J=\{\{j_1\},\ldots,\{j_{m/2}\}\}$, $m/2+1\le j_1<j_2<\cdots<j_{m/2}\le n$ we define the set of graphs $G_{J,i}(V,E_{J,i})$ where $E_{J,i}:=\{ \{i\}\times([m/2]\backslash \{i\})\cup\{\{i,j_1\},\ldots,\{i,j_{m/2}\}\}\ | \ m/2+1\le j_1<j_2<\cdots<j_{m/2}\le n\}.$ Notice that any query $Q$ in $A_s\backslash A_{s,i}$ gives no information about the vertices $j_1,\ldots,j_{m/2}$. That is because, either $i\not\in Q$ and then the answer is ``NO'' or $i\in Q$ and $Q\cap ([m/2]\backslash\{i\}) \not=\emptyset$ and then the answer is ``YES''.

We now give an algorithm ${\cal B}$ that learns $m/2$-Loop over $n-m/2$ vertices that are labeled with $\{m/2,\ldots,n\}$ with success probability at least $1/2$ using at most $4cm\log n$ queries. This gives a contradiction to the result of Lemma~\ref{RNLL} and then the result follows.

Algorithm ${\cal B}$ chooses a random uniform $i\in [m/2]$ and runs algorithm ${\cal A}$. Suppose the set of $m/2$ loops of the target is $J=\{\{j_1\},\ldots,\{j_{m/2}\}\}$. The goal of algorithm ${\cal B}$ is to provide ${\cal A}$ answers to its queries as if the target is $G_{J,i}$. Therefore, for each query $Q$ that ${\cal A}$ asks, if  $i\not\in Q$ then algorithm ${\cal B}$ returns the answer ``NO'' to ${\cal A}$ and if $i\in Q$ and $Q\cap ([m/2]\backslash\{i\}) \not=\emptyset$ then it returns the answer ``YES''. If $Q\cap [m/2]=\{i\}$ then algorithm ${\cal B}$  asks the query $Q\cup\{i\}$ and returns the answer to ${\cal A}$. Algorithm ${\cal B}$ halts if the number of queries is more than $8cm\log n$ or ${\cal A}$ outputs a graph $H(V,E')$. If $E'=\{i\}\times([m/2]\backslash \{i\})\cup\{\{i,j'_1\},\ldots,\{i,j'_{m/2}\}\}$ then algorithm ${\cal B}$ outputs $G(V,\{\{j'_1\},\ldots,\{j'_{m/2}\}\})$ otherwise it returns an empty graph.

Algorithm ${\cal B}$ fails if the number of queries is greater than $8cm\log n$ or algorithm ${\cal A}$ fails. The former happens with probability at most $1/4$ and the latter with probability at most $1/4$. Therefore the error of the algorithm is at most $1/2$.
\end{proof}

\section{Two and Three Round Randomized Learning - \texorpdfstring{$m$}{} is Known}
In this section we study two-round randomized algorithms.

We first prove that there is a two-round Monte Carlo randomized learning algorithm (with $1/poly(n)$ error probability) for $m$-Graph that asks $O(m^{4/3}\log n)$ queries. Then we show that, for $n>m^m$, there is a two-round randomized Monte Carlo learning algorithm for $m$-Graph that asks $O(m\log n)$ queries.

For Las Vegas algorithm we prove the above query complexities for three-round algorithms. We then show that any two-round Las Vegas algorithm must ask at least $\Omega((m^{4/3}\log^{1/3}n)/(\log^{1/3} m))$ queries. For $m>(\log n)^{\omega(1)}$ this lower bound is $\Omega(m^{4/3-o(1)}\log n)$.

\subsection{Learning the Neighbours in an Independent Set}

\begin{figure}[t!]
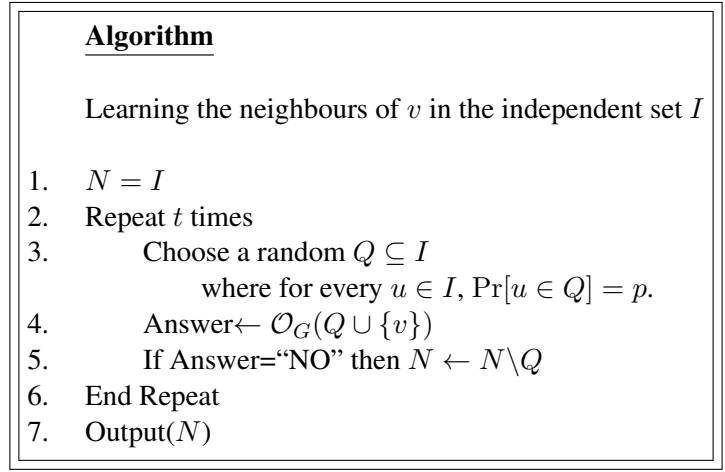

  \begin{center}
   \fbox{\fbox{\begin{minipage}{28em}
  \begin{tabbing}
  xxxx\=xxxx\=xxxx\=xxxx\=xxxx\= \kill
  \>\underline{{\bf Algorithm}}\\ \\
  \> Learning the neighbours of $v$ in the independent set $I$\\ \\
  1. \>$N=I$\\
  2. \> Repeat $t$ times\\
  3. \>\> Choose a random $Q\subseteq I$\\
  \>\>\>where for every $u\in I$, $\Pr[u\in Q]=p$.\\
  4. \>\> Answer$\gets{\cal O}_G(Q\cup \{v\})$\\
  5. \>\> If Answer=``NO'' then $N\gets N\backslash Q$\\
  6. \> End Repeat\\
  7. \> Output($N$)
  \end{tabbing}\end{minipage}}}
  \end{center}
	\caption{An algorithm that given an independent set $I$ in $V$, finds the vertices in $I$ that are neighbours of $v$.}
	\label{AlgN}
	\end{figure}

\begin{restatable}{lem}{ilem}\label{FinfN} Consider the Algorithm in Figure~\ref{AlgN}. Let $I\subset V$ be an independent set in $G$. Let $v\not\in I$ be a vertex in $G$. For $p=1/m$ and $t=4m(\ln n+\ln(1/\delta))$ with probability at least $1-\delta$ the output $N$ of the algorithm satisfies $N=\Gamma(v)\cap I$. That is, $N$ contains only the neighbours of $v$ in $I$.
\end{restatable}

\begin{proof} The output $N$ is not the set of neighbors of $v$ if and only if for some $u\not\in \Gamma(v)$ each query $Q$ in the algorithm satisfies: $u\not\in Q$ or $\Gamma(v)\cap Q\not=\emptyset$. Therefore, the probability that output $N$ is not the set of neighbors of $v$ is less than
$$n(1-p(1-p)^m)^t\le n\left(1-\frac{1}{4m}\right)^t\le \delta.$$
\end{proof}

\subsection{Upper Bound for Randomized Two-Round Algorithm}\label{FactS}
Consider the algorithm in Figure~\ref{Alg1} with $p<1/(8\sqrt{m})$ and $r=1/(2p)$. By Lemma~\ref{main}, $E\subseteq E(H)$ and for $t=O((1/p)^2(\log n+\log(1/\delta)))$, with probability at least $1-\delta$, every edge $\{u,v\}\in E(H)\backslash E$ satisfies $\deg_G(u)+\deg_G(v)>r+1$. Assume for now that this is true with probability $1$.
\begin{fact}~\label{F1} Every edge $\{u,v\}\in E(H)\backslash E$ satisfies $\deg_G(u)+\deg_G(v)>r+1$.
\end{fact}

Figure~\ref{Proof1} (in the appendix section) help you to follow the proof.
Let $V_G:=\{v| \{u,v\}\in E\}$, $V_{H}:=\{v| \{u,v\}\in E(H)\}$.
We partition the set of edges $E(H)\backslash E$ to three disjoint set $E_0\cup E_1\cup E_2$ where $E_{i}=\{\{u,v\}\in E(H)\backslash E\ :\ |\{u,v\}\cap V_G|=i\}$. Fact~\ref{F1} immediately implies that
\begin{fact}~\label{F11} $E_0=\emptyset$. That is, every edge in $E(H)\backslash E$, at least one of its endpoints is in $V_G$.
\end{fact}

Let $u\in V_{H}\backslash V_G$. Then $\deg_G(u)=0$. Therefore, by Fact~\ref{F1},

\begin{fact}~\label{star} For any edge $\{u,v\}\in E_1$ where $u\in V_{H}\backslash V_G$ we have $\deg_G(v)>r+1$.
\end{fact}
Since the number of vertices in $G$ that have degree greater than $r+1$ is less than $2m/(r+2)\le r/8$, the degree of each $u\in V_H\backslash V_G$ in the graph $H$ is at most $r/8$.
\begin{fact}~\label{F2} If $u\in V_H\backslash V_G$ then $\deg_H(u)\le r/8$. In particular, all the vertices of degree greater than $r/8$ are in $V_G$.
\end{fact}

Now take any edge $\{u',v'\}\in E(H)\backslash E$. By Fact~\ref{F1}, $\deg_G(u')+\deg_G(v')>r+1$ and therefore either $\deg_G(u')> r/2$ or $\deg_G(v')>r/2$. If $\deg_G(v')>r/2$ then $\deg_H(v')>r/2$ and by Fact~\ref{F2}, $v'\in V_G$. This with Fact~\ref{star} shows that
\begin{fact}~\label{F3}
Every edge $\{u',v'\}\in E_2$, one of its endpoints is in $V_G$ and has degree at least $r/2$ in $G$ and therefore also in $H$.

Every edge $\{u',v'\}\in E_1$ has one endpoint in $V_G$, and a degree greater than $r+1$ in $G$ and therefore also in $H$.
\end{fact}
Denote by $W$ the set of all vertices of degree greater than $r/2$ in $H$.
Then
\begin{fact}~\label{F4} The number of vertices of degree more than $r/2$ in $H$ is at most $8m/r$. That is, $|W|\le 8m/r$.
\end{fact}
\begin{proof}  Let $u$ be a vertex in $G$ of degree less than $r/4$. Then all its edges in $H$ are in $E_2\cup E$. This is because if it has an edge in $E_1$ then by Fact~\ref{F3}, its degree in $G$ is more than $r+1$. If $\{u,v\}\in E_2$ then by Fact~\ref{F3} the degree of $v$ in $G$ is at least $r/2$. Since the number of vertices in $G$ of degree at least $r/2$ is at most $2m/(r/2)< r/4$ the degree of $u$ in $H$ is less than $r/4+r/4=r/2$. This shows that a vertex in $H$ of degree at least $r/2$ is of degree at least $r/4$ in $G$. Therefore the number of such vertices is at most $2m/(r/4)= 8m/r$.
\end{proof}

We now prove
\begin{fact}~\label{F5} $$|E_2|\le \frac{8m^2}{r}.$$
\end{fact}
\begin{proof} By Fact~\ref{F3}, one of its endpoint vertices of an edges in $E_2$ is in $V_G$ and has degree at least $r/2$ in $G$. There are at most $4m/r$ such vertices and each one can have at most $|V_G|\le 2m$ edges.
\end{proof}

Now define for every vertex $w\in W$ the set $I_w$ that contains all the neighbours $u\in \Gamma_H(w)$ where $\deg_H(u)\le r/8$ and $\Gamma_H(u)\cap \Gamma_H(w)$ contains only vertices of degree more than $r+1$ in $H$.
\begin{fact}~\label{F6} $I_w$ is an independent set.
\end{fact}
\begin{proof} If $I_w$ contains $\{u,v\}\in E(H)$ then $\deg_H(u),\deg_H(v)<r/8$. Since $\Gamma_H(u)\cap \Gamma_H(w)$ contains $v$, we get $\deg_H(v)>r+1$. This gives a contradiction.
\end{proof}

We now show
\begin{fact}~\label{F7} We have
$$E_1\subseteq E_W:=\bigcup_{w\in W} \{\{w,u\}\ |\ u\in I_w\}$$
\end{fact}
\begin{proof} If $\{w,u\}\in E_1$ then, by the definition of $E_1$ and Fact~\ref{F3} and~\ref{F2}, one of the vertices, say $w$, is in $V_G$ and is in $W$ and the other vertex, $u$, is in $V_H\backslash V_G$ and has degree at most $r/8$ in $H$. Now we show that $u\in I_w$. If $u\not\in I_w$ then $\Gamma_H(u)\cap\Gamma_H(w)$ contains a vertex $v$ of degree less or equal to $r+1$. If $v\in V_H\backslash V_G$ then $\{u,v\}$ is an edge in $E_0$ and we get a contradiction to Fact~\ref{F11}. Therefore $v\in V_G$ then since $\{v,u\}$ is an edge in $E_1$ and $\deg_H(u)\le r/8$ we must have $\deg_H(v)\ge r+1$ and again we get a contradiction.
\end{proof}

Since $U:=E(H)\backslash E_W\subseteq E\cup E_2$, by Fact~\ref{F5}, we get
\begin{fact}~\label{F8}
$$|U|=|E(H)\backslash E_W|\le m+\frac{8m^2}{r}.$$
\end{fact}
Therefore, we first find $W$ and for each $w\in W$ learns the neighbours of $w$ in $I_w$. This eliminates all the edges of $w$ that are not in the graph. In particular, it removes all the edges in $E_1$ and some of those in $E_2$. Then for each edge in $U=E(H)\backslash E_W$ we ask a query.

We now can prove

\begin{theorem}\label{MCTR} There is a two-round Monte Carlo randomized learning algorithm with $1/poly(n)$ error probability for $m$-Graph that asks $O(m^{4/3}\log n)$ queries.

There is a three-round Las Vegas randomized learning algorithm for $m$-Graph that asks $O(m^{4/3}\log n)$ expected number of queries.
\end{theorem}
\begin{proof} Consider the First round in Figure~\ref{Alg1} and the second round in Figure~\ref{Alg2}.
We choose $p=1/m^{2/3}$ and $\delta=1/poly(n)$. Then $r=m^{2/3}/2$ and $t=O(m^{4/3}\log n)$.

By Fact~\ref{F4}, $|W|\le 8m/r$. By Lemma~\ref{FinfN} finding the neighbours of each $w\in W$ takes $O(m\log n)$ queries. Therefore the total number of queries for steps 1-5 in the algorithm is $O((m^2/r)\log n)=O(m^{4/3}\log n)$. By Fact~\ref{F8} the number of queries in steps 6-8 is at most $m+8m^2/r=O(m^{4/3})$.
\end{proof}

\begin{figure}[t!]
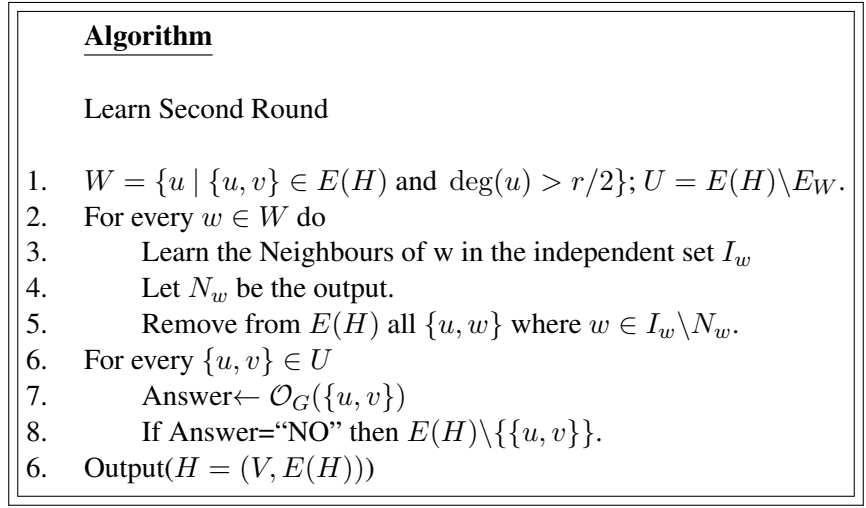

  \begin{center}
   \fbox{\fbox{\begin{minipage}{28em}
  \begin{tabbing}
  xxxx\=xxxx\=xxxx\=xxxx\=xxxx\= \kill
  \>\underline{{\bf Algorithm}}\\ \\
  \> Learn Second Round\\ \\
  1. \>$W=\{u\ |\ \{u,v\}\in E(H)\mbox{\ and\ } \deg(u)>r/2\}$; $U=E(H)\backslash E_W$.\\
  2. \> For every $w\in W$ do\\
  3. \>\> Learn the Neighbours of w in the independent set $I_w$\\
  4. \>\>Let $N_w$ be the output.\\
  5. \>\>Remove from $E(H)$ all $\{u,w\}$ where $w\in I_w\backslash N_w$.\\
  6. \> For every $\{u,v\}\in U$ \\
  7. \>\> Answer$\gets {\cal O}_G(\{u,v\})$\\
  8. \>\> If Answer=\NO\ then $E(H)\backslash \{\{u,v\}\}$.\\
  6. \> Output($H=(V,E(H))$)
  \end{tabbing}\end{minipage}}}
  \end{center}
	\caption{Second Round in the Algorithm.}
	\label{Alg2}
	\end{figure}

\subsection{Randomized Three Rounds Monte Carlo Algorithm}
In this section we show that when $m$ is known to the algorithm then there is a three-round Monte Carlo algorithm that asks $O(m\log n)$ queries.

In~\cite{AC08}, Angluin and Chen gave a three-round Monte Carlo randomized algorithm that asks $O(m\log n+\sqrt{m}\log^2n)$ queries. We give here a three-round algorithm that asks $O(m\log n +m^{1.5})$ queries. Both results imply
\begin{theorem}\label{MCTR2} There is a three-round Monte Carlo randomized learning algorithm with $1/poly(n)$ error probability for $m$-Graph that asks $O(m\log n)$ queries.

There is a four-round Las Vegas randomized learning algorithm for $m$-Graph that asks $O(m\log n)$ expected number of queries.
\end{theorem}
\begin{proof}
If $m\ge \log^{2}n$ then $O(m\log n+\sqrt{m}\log^2n)=O(m\log n)$. Otherwise, $m< \log^2n$ and then $O(m\log n +m^{1.5})=O(m\log n)$.

Now we describe the algorithm. In the first round we run the algorithm in Figure~\ref{Alg1} with $p=1/(16\sqrt{m})$. By the facts in Section~\ref{FactS}, for $r=8\sqrt{m}$ we have
\begin{enumerate}
\item All the edges $\{u,v\}\in E(H)\backslash E$ satisfies $\deg_G(u)+\deg_G(v)> 8\sqrt{m}+1$.
\item $E_0=\emptyset$, $|E_2|\le m^{1.5}+1$, $E_1\subset E_W$, $|W|\le \sqrt{m}$ and $|E(H)\backslash E_W|\le m^{1.5}+m+1$.
\end{enumerate}
In the first round we ask $O(1/p^2)(\log n +\log (1/\delta))=O(m\log n)$ queries. Now for each $w\in W$ we need to find the neighbors $\Gamma_G(w)\cap I_w$ of $w$~in~$G$. If we do that using the previous algorithm we get $O(m^{1.5}\log n)$ queries. Instead, we will add another round that estimates the number of neighbours of each $w\in W$ in~$G$ (and in $I_w$). We then learn the neighbour with $O(deg_G(w)\log n)$. This is possible because $I_w$ is ab independent set. The estimation can be done by doubling and estimating with Chernoff bound. See~\cite{F16}. The estimation can be done with $O(\log m \log n)$ queries for each $w\in W$ and success probability $1-1/poly(n)$ and therefore with $O(|W|\log m\log n)=O(\sqrt{m}\log m \log n)$ queries. Then since $|E(H)\backslash E_W|\le 2m^{1.5}+m+1$ finding the other edges in the graph can be done in (round 2 or 3) with $O(m^{1.5})$ queries. The total number of queries is $O(m\log n+m^{1.5})$.
\end{proof}

\subsection{Two-Round Learning for Large \texorpdfstring{$n$}{}}
In this section we prove
\begin{restatable}{theorem}{sthm}\label{MCTRN} Let $w>m$. There is a two-round randomized Monte Carlo learning algorithm with $1/w^{O(1)}$
error probability for $m$-Graph that asks $O(m^2 \log w+m\log n)$.

In particular, when $w=n^{c/m}$ for any constant $c$ (and therefore $m<\log n$) the algorithm asks $O(m\log n)$ queries.
\end{restatable}

\begin{proof} We first partition the set of vertices into $u=poly(w)$ sets $V_1,\ldots,V_u$. The probability that each set contains at most one vertex of degree not equal zero is at least
\begin{eqnarray}\label{ran034}
\left(1-\frac{1}{u}\right)\left(1-\frac{2}{u}\right)\cdots \left(1-\frac{2m-1}{u}\right)
\ge 1-\frac{m(2m-1)}{u}\ge 1-\frac{1}{w^{O(1)}}.
\end{eqnarray}
Assuming the vertices that have degree greater than zero are in different sets, we learn the graph over the $u$ sets in one round with probability at least $1-1/w^{O(1)}$ and
$O(m^2\log u)=O(m^2\log w)$ queries using Theorem~\ref{UBNAMC}. That is, we assume that each set $V_j$ is one vertex and we learn the graph over the sets $V_1,\ldots,V_u$. Then, for each query $Q\subseteq [u]$ the algorithm asks the query $\cup_{w\in Q}V_w$. When the algorithm discover an edge $\{V_i,V_j\}$ then it knows that there is an edge between one of the vertices in $V_i$ with one of the vertices in $V_j$. We will call the edge $\{V_i,V_j\}$ a {\it set edge} and $V_i$ a {\it set vertex}.

Now, suppose there is a set of edges $e$ with the set of endpoints vertices $V_i$ and $V_j$. We can learn the endpoints vertices of $e$ deterministically with $O(\log n)$ queries. To learn the endpoint in $V_j$ the algorithm considers $V_i$ as one vertex and runs the algorithm that learn $1$-Loop in $V_j$. That is, for each query $Q\subseteq V_j$ the algorithm asks the query $V_i\cup Q$. Therefore, in the second round, we can deterministically learn the endpoints of all the edges in $O(m\log n)$ queries.
\end{proof}

We now convert the above algorithm to a three-round Las Vegas algorithm.

We first give one definition and a lemma. For a query $Q$ and a vertex $u$ we define $[u\in Q]=1$ if $u\in Q$ and $[u\in Q]=0$ if $u\not\in Q$. We now prove
\begin{lem}\label{oneor} There is a non-adaptive deterministic algorithm that asks $O(\log n)$ queries and if the target is $0$-Loop or $1$-Loop then it learns the target and if
it is $k$-Loop, $k>1$, then it returns an ``{\rm ERROR}''.
\end{lem}
\begin{proof} We define the set of $2t$ queries $\{Q_1,\ldots,Q_{2t}\}$ as follows. Each vertex $i$ appears in
exactly $t$ queries and no two vertices appears in the same set of queries. We must have ${\binom{2t}{t}}\ge n$ and therefore it is enough to take $2t=\log n+ 2\log\log n=O(\log n)$.

Now if the target is $0$-Loop then the vector of answers is $([i\in Q_1],\ldots,[i\in Q_{2t}])$ is the zero vector.
If the target is $1$-Loop, $\{i\}$, then the vector of answers is $([i\in Q_1],\ldots,[i\in Q_{2t}])$ which uniquely determines $i$. This vector contains $t$ ones and $t$ zeros. When the target is $k$-Loop, $L=\{\{i_1\},\ldots,\{i_k\}\}$, $k>1$, then the vector of answers is $\vee_{j=1}^k([i_j\in Q_1],\ldots, [i_j\in Q_{2t}])$ (bitwise or) which contains at least $t+1$ ones.
This indicates that the target is a $k$-Loop for some $k>1$.
\end{proof}

\begin{restatable}{theorem}{tthrm}\label{LVTRN} There is a three-round randomized Las Vegas learning algorithm for $m$-Graph that asks $O(m^2 \log m+m\log n)$ expected number of queries.

In particular, when $n\ge m^m$, the algorithm asks $O(m\log n)$ queries.
\end{restatable}

\begin{proof} We run the algorithm in the proof of Theorem~\ref{MCTRN} but in the second round
we use the algorithm in Lemma~\ref{oneor} for learning the endpoints vertices.
If the algorithm fails at some round we run the deterministic algorithm that asks $O(m^3\log n)$ queries in the third round. We now give more details.

We first partition the set of vertices into $u=2m^4(2m-1)$ sets $V_1,\ldots,V_u$. By (\ref{ran034}), the probability that each set contains at most one vertex of degree not equal zero is at least $1-1/m^3$. Assuming success, by the proof of Theorem~\ref{UBNAMC}, the first round finds all the edges and an expected of $1/poly(u)$ more edges that are not in the target.
The edges that are not eliminated are found by the deterministic algorithm in the second round.

Suppose the first round fails to distribute the vertices of degree not equal zero in different sets. We show how to discover that in the second round. We distinguish between two cases: The first case is when there is an edge $\{u,v\}$ where $\{u,v\}\subseteq V_i$ for some $i$. The second case is when there is no edge between two vertices in the same set but there is at least two nodes $u$ and $v$ of degrees not equal to zero in the same set $V_i$. One of those two cases happens with probability less than $1/m^3$.

If the first case happens, the algorithm in the first round will not be able to eliminate any of the set edges $\{V_i,V_j\}$ for all $V_j$. This is because when the set vertex $V_i$ is in the query the answer is always ``YES''. Therefore, at the end of the first stage, there will be at least $u>m$ set edges that are not eliminated. That is, all the set edges $\{V_i,V_j\}$ for all $V_j$. Then the algorithm knows that the first case happens and it runs the deterministic algorithm that asks $m^3\log n$ queries in the second round.

Now suppose the second case happens. Suppose $V_j$ contains two vertices $v_1$ and $v_2$ of non-zero degree (with no edge between them). Let $\{u_1,v_1\}$ and $\{u_2,v_2\}$ be two edges in $G$. We here again distinguish between two subcases. The first subcase is when $u_1,u_2\in V_k$, $k\not=j$. The second subcase is when $u_1\in V_{k_1}$ and $u_2\in V_{k_2}$, and $k_1,k_2,j$ are distinct. In the first subcase, when the algorithm consider the set $V_k$ as one vertex and runs the algorithm in Lemma~\ref{oneor}, the algorithm output ``ERROR'' and then it knows that this subcase happens. In the second subcase, when the algorithm consider the set $V_{k_1}$ as one vertex and runs the algorithm in Lemma~\ref{oneor}, it learns $v_1$ and when the algorithm consider the set $V_{k_2}$ as one vertex it learns $v_2\not= v_1$. Then the algorithm knows that the second subcase happens. When the second case happens the algorithm runs the deterministic algorithm that asks $m^3\log n$ queries in the third round.

The expected number of queries is
$$\left(1-\frac{1}{m^3}\right) (m^2\log u+m\log n)+\frac{1}{m^3}m^3\log n=
O(m^2 \log m+m\log n).$$
\end{proof}

\subsection{Lower Bound for Las Vegas Randomized Algorithm}
In this section we prove
\begin{theorem}\label{LVLBTR} Let $\log n<m=o(\sqrt{n})$. Any two-round Las Vegas learning algorithm for $m$-Graph must ask at least $$\Omega\left(\frac{m^{4/3}\log^{1/3} n}{\log^{1/3}m}\right)$$ queries on average.
\end{theorem}
\begin{proof} Let $A$ be a two-round Las Vegas learning algorithm for $m$-Graph. Notice that since the algorithm succeeds with probability 1, the second round must be deterministic.

We define a distribution $D$ over the targets as follows: we first choose $r=m/2$ (fixed) distinct vertices $V'=\{v_1,\ldots, v_r\}$. Then randomly and uniformly choose $1\le t\le r$. Then randomly uniformly
choose $s=m/2-d$ distinct vertices $U=\{u_1,\cdots,u_s\}$ where $U\cap V'=\emptyset$ and $d=(m^{2/3}\log^{1/3}m)/(2^{10}\log^{1/3}n)$. Then randomly uniformly
choose $d$ (not necessarily distinct) vertices $W=\{w_1,\ldots,w_d\}\subseteq V\backslash (V'\cup U)$. Then define the target $$T=\{\{v_t,v_j\}\ |\ j\not=t; j=1,\ldots,r\} \cup \{\{v_t,u_j\},\{v_t,w_k\} \ |\ j=1,\ldots,s; \ k=1,\ldots,d\}.$$

Now let $X_A(y,{\cal O}_I)$ be a random variable that is the number of queries asked by the algorithm $A$
with a seed $y$ and target $I$. If for any deterministic two-round algorithm $B$ we have $\E_{I\in D}[X_B({\cal O}_I)]\ge q$ where $X_B({\cal O}_I)$ is the number of queries asked by $B$ then $\E_{I\in D}[X_A(y,{\cal O}_I)|y]\ge q$ and then the query complexity of $A$ is
$$\max_I\E_y[X_A(y,{\cal O}_I)]\ge \E_{I\in D}\E_y[X_A(y,{\cal O}_I)]= \E_y\E_{I\in D}[X_A(y,{\cal O}_I)|y]\ge \E_y[q]=q.$$
Therefore, what remains to prove is that for any two round deterministic algorithm $B$ we have $\E_{I\in D}[X_B({\cal O}_I)]\ge q$.

Consider the first round of $B$ with $q=(m^{4/3}\log^{1/3}n)/\log^{1/3}m$ queries $Q_1,\ldots,Q_q$. Consider the set of queries $S_i=\{Q_i\ |\ Q_i
\cap V'=\{v_i\}\}$ for $i=1,\ldots,r$. Then $r\E_i[|S_i|]\le q$ and therefore at least $7/8$ of the $i\in [s]$ satisfies $|S_i|\le 8q/r$. Therefore, with probability at least $7/8$ we have $|S_t|\le 8q/r$. Suppose we have chosen such $t$ and after the first round we provide the algorithm~$v_t$. Therefore, the algorithm only needs to learn $U$ and $W$. We will show next that after the first round even if we provide the learner $U$ there will still be many vertices about which no information is known, with high probability, $W$.

For the ease of notation we write $Q(I)=1$ if $Q\cap I\not=\emptyset$ and $0$ otherwise.
Now every query $Q_i$ that satisfies $Q_i \cap V'\not=\{v_t\}$ will give no information about $u_i$ or $w_i$. Therefore, the queries that are relevant to learning are only the
queries in~$S_t$. Let $S_t=\{Q'_1,\ldots,Q'_\ell\}$. Since $\ell=|S_t|\le 8q/r$, the algorithm can get at most $2^\ell\le 2^{8q/r}$ possible answers.
For each vector of $\ell$ possible answers $a=(a_1,\ldots,a_\ell)\in\{0,1\}^\ell$ to the queries in $S_t$ we define ${\cal I}_a=\{I\subseteq V\backslash V'\ :\ |I|=s, (Q_1'(I),\ldots,Q_\ell'(I))=a\}$. Let $U\in {\cal I}_{a'}$, i.e., $(Q_1'(U),\ldots,Q_\ell'(U))=a'$.
Since $$T:=\sum_{a\in\{0,1\}^\ell}|{\cal I}_a|= {\binom{n-r}{s}}, \sum_{a,|{\cal I}_a|\le T/2^{\ell+3}}|{\cal I}_a|\le  \frac{{\binom{n-r}{s}}}{8},$$
with probability at least $7/8$, $|{\cal I}_{a'}|\ge T/2^{\ell+3}.$ Suppose the latter statement is true with probability $1$. Let $Z_{a'}=\cup_{I\in {\cal I}_{a'}}I$. Notice that for every $w\in Z_{a'}$ there is $I\in {\cal I}_{a'}$ such that $w\in I$ and therefore (bitwise) $(Q_1'({w}),\ldots,Q_\ell'({w}))\le (Q_1'(I),\ldots,Q_\ell'(I))=a'$ which implies that if $(Q_1'(U\cup W),\ldots,Q_\ell'(U\cup W))=a'$ no information is known about the vertices in $Z_{a'}$ after the first round. If this happen then there are $|Z_{a'}|$ vertices where no information is known about them. We next prove that with high probability $W\subseteq Z_{a'}$ and therefore $(Q_1'(U\cup W),\ldots,Q_\ell'(U\cup W))=a'$. Then in the next round we must run a deterministic algorithm that learns the $d$ vertices $W$ in $Z_{a'}$. This, by Lemma~\ref{D2R1}, requires at least
\begin{eqnarray}\label{uhiruh}
\frac{d^2\log |Z_{a'}|}{\log d}\end{eqnarray} queries.

Now we estimate $|Z_{a'}|$ and prove that with high probability we have $W\subseteq Z_{a'}$.
We have  $$ |{\cal I}_{a'}|\ge\frac{T}{2^{\ell+3}}\ge  \frac{{\binom{n-r}{ s}}}{2^{\ell+3}}\ge \frac{(n-r)^s}{2^{\ell+3}s!}\left(1-\frac{s(s-1)}{2(n-r)}\right).$$
On the other hand $$|{\cal I}_{a'}|\le {\binom{|Z_{a'}|}{ s}}\le \frac{|Z_{a'}|^s}{s!}.$$
Therefore,
$$|Z_{a'}|\ge \frac{n-r}{2^{(\ell+3)/s}}\left(1-\frac{s(s-1)}{2(n-r)}\right)^{1/s}=\frac{n}{2^{(\ell+3)/s}}(1-o(n)).$$
The probability that $|W|=d$ and $W\subseteq  Z_{a'}$ is at least
$$\left(1-\frac{d(d-1)}{n-r-s}\right)\left(\frac{|Z_{a'}|-s}{n-s-r}\right)^d\ge \frac{1}{2^{(\ell+3)d/s}}(1-o(1))\ge \frac{7}{8}.$$
Therefore, if we provide the algorithm $U$ after the first round, with probability at least
$7/8$ no information is known about $W$. All the above is true with probability at least $1/2$.
By (\ref{uhiruh}), in the second round the algorithm needs to ask at least
$$\frac{d^2\log |Z_{a'}|}{\log d}=\Omega\left(\frac{m^{4/3}\log^{1/3} n}{\log^{1/3}m}\right)$$ queries.
\end{proof}

\section{Deterministic Algorithms when \texorpdfstring{$m$}{} is Known}
In this section we study deterministic algorithms for learning $m$-graph when $m$ is known.
It is known that any non-adaptive deterministic algorithm must ask at least $\Omega((m^3\log n)/\log m)$, \cite{DR82}, and there is a polynomial time non-adaptive algorithm that asks $O(m^3\log n)$, \cite{B15}. In this section we prove (non-constructively) that there is a two-round algorithm that asks $O(m^2\log n)$ queries and then give the lower bound $\Omega((m^2\log n)/\log m)$ for the number of queries. Finally, we give a four-round deterministic algorithm that asks $O(m^{2+\epsilon}\log n)$ queries for any constant $\epsilon$.

\subsection{Nonconstructive Upper Bound for Two-Round Deterministic Algorithm}

We first prove the following three results.
\begin{lem}\label{LD01} Let $E=\{e_1,\ldots,e_t\}$ be a set of edges. Let $Q\subseteq V$ be a random query where for each vertex $i\in V$, $i$ is included in $Q$ independently with probability~$p$. Then
$$\Pr[(\exists i\in [t]) e_i\subseteq Q]\ge p\cdot (1-(1-p)^{t}).$$
\end{lem}
\begin{proof} Define the event $A_{k,j}=[(\exists i\in [k,j]) e_i\subseteq Q]$. We prove the result by induction on $t$. For $t=1$ we have one edge $e$ and $\Pr[e\subseteq Q]=p^2\ge p\cdot (1-(1-p)^{1})$. Let $u\in e_1$.
Assume w.l.o.g that $u\in e_i$ for $i\in [\ell]$ and $u\not\in e_i$ for $i\in[\ell+1,t]$. Define the event $B=[(\exists i\in [\ell]) (e_i\backslash \{u\})\subseteq Q$]. Then, by the induction hypothesis,
\begin{eqnarray*}
\Pr[A_{1,t}]&=&\Pr[u\in Q]\Pr [A_{1,t}|u\in Q]+\Pr[u\not\in Q]\Pr[A_{1,t}|u\not\in Q]\\
&=& p\Pr[B\vee A_{\ell+1,t}]+(1-p)\Pr[A_{\ell+1,t}]\\
&\ge & p\Pr[B]+(1-p)p(1-(1-p)^{t-\ell})\\
&=& p(1-(1-p)^\ell)+(1-p)p(1-(1-p)^{t-\ell})\\
&=& p(1+(1-p)(1-(1-p)^{\ell-1})-(1-p)^{t-\ell+1})\\
&\ge& p(1+(1-p)^{t-\ell+1}(1-(1-p)^{\ell-1})-(1-p)^{t-\ell+1})\\
&=&p\cdot (1-(1-p)^{t}).
\end{eqnarray*}
\end{proof}

\begin{lem}\label{LD02} Let $E=\{e_1,\ldots,e_r\}$ and $E'=\{e_1',\ldots,e_t'\}$ be two disjoint sets of edges. Let $Q\subseteq V$ be a random query where for each vertex $i\in V$, $i$ is included in $Q$ independently with probability~$p$. Then
$$\Pr[(\forall i\in [r])\ e_i\not\subseteq Q \ |\ (\exists j\in [t]) e_j'\subseteq Q]\ge (1-p)^{r}.$$
\end{lem}
\begin{proof} The proof is by induction on the number of edges in $E$. Define the events $B=[(\forall i\in [r])\ e_i\not\subseteq Q]$ and $A=[(\exists j\in [t]) e_i'\subseteq Q]$.
Assume w.l.o.g $u\in e_i$ for $i\in [\ell]$, $u\not\in e_i$ for $i\in[\ell+1,r]$, $u\in e_i'$ for $i\in [\ell']$ and $u\not\in e_i'$ for $i\in[\ell'+1,t]$. Here, $\ell\ge 1$ and $\ell'\ge 0$. Then
$\Pr[B|A]\ge \Pr[u\not\in Q]\cdot \Pr[B|A\mbox{\ and\ }u\not\in Q]=(1-p) \cdot \Pr[B|A\mbox{\ and\ }u\not\in Q].$ If $u\not\in Q$ then $e_i\not\subseteq Q$ for all $i\in [\ell]$ and then $B=B':=[(\forall j\in [\ell+1,r])e_j\not\subseteq Q]$. Also, $e_i'\not\subseteq Q$ for all $i\in [\ell']$ and the event $[A\mbox{\ and\ }u\not\in Q]$ is equivalent to $[(A'\mbox{\ and\ }u\not\in Q)]$ where $A':= [\exists i\in [\ell'+1,t])\ e_i'\subseteq Q]$. Therefore
$\Pr[B|A]=\Pr[B'|A']$. By the induction hypothesis $\Pr[B'|A']\ge (1-p)^{r-\ell}$ and therefore $\Pr[B|A]\ge (1-p)^{r-\ell+1}\ge (1-p)^r$.
\end{proof}

\begin{lem}\label{Donequery} Let $G=(V,E)$ be the target graph with $n$ vertices and $m$ edges.
Let $Q\subseteq V$ be a random query where for each vertex $i\in V$, $i$ is included in $Q$ independently with probability~$p$. Let $e_1',\ldots,e_t'\not\in E$. Suppose the probability that there is $i\in [t]$ such that $e_i'\subseteq Q$ is at least $q$. Then
$$\Pr[(\exists i\in [t])\ e_i'\subseteq Q \mbox{\ and \ } {\cal O}_G(Q)=\NO]\ge q(1-p)^{m}.$$
In particular,
$$\Pr[(\exists i\in [t])\ e_i'\subseteq Q \mbox{\ and \ } {\cal O}_G(Q)=\NO]\ge p\cdot (1-(1-p)^t)(1-p)^{m}.$$
\end{lem}
\begin{proof} Let $A:= [(\exists i\in [t])\ e_i'\subseteq Q]$. Let $E=\{e_1,\ldots,e_m\}$.
Then the event ${\cal O}_G(Q)=\NO$ is equivalent to the even $B:=[(\forall j\in [m])e_j\not\subseteq Q]$.
Now by Lemma~\ref{LD01} and \ref{LD02} we have
\begin{eqnarray*}
\Pr[A\mbox{\ and\ }B]&=&\Pr[A]\cdot\Pr[B|A]\ge q(1-p)^m\ge p\cdot (1-(1-p)^t)(1-p)^{m}.
\end{eqnarray*}
\end{proof}

We are now ready to prove our main result

\begin{restatable}{theorem}{ethm}\label{UBDTR}

There is a two-round deterministic learning algorithm for $m$-Graph that asks $t=O(m^2\log n)$ queries.
\end{restatable}

\begin{proof} We will first show that there is a set of $t$ queries $Q_1,\ldots,Q_t$ that satisfies: For every graph $G=([n],E=\{e_1,\ldots,e_m\})$ with $m$ edges and for every set of $m$ edges $E'=\{e_1',\ldots,e_m'\}$ not in $G$, there is a query $Q_j$ such that ${\cal O}_G(Q_j)=\NO$ and $(\exists i\in [m])\ e_i'\subseteq Q_j$.

Choose $Q_1,\ldots,Q_t$ where each $Q_j\subseteq V$ is a random query where for each vertex $i\in [n]$, $i$ is included in $Q$ independently with probability~$1/m$. By Lemma~\ref{Donequery}, probability that above event is not true for some graph $G$ is at most
$${\binom{{\binom{n}{ 2}}}{m}}^2 \left(1-\frac{1}{m}\left(1-\left(1-\frac{1}{m}\right)^m\right)\left(1-\frac{1}{m}\right)^m\right)^t.$$
This is less than $1$ for $t=O(m^2\log n)$.  Therefore there are such queries.

In the first round, the algorithm defines $E(H)=\{\{u,v\}|u,v\in V,u\not=v\}$ and asks all the queries $Q_1,\ldots,Q_t$. For each query $Q_j$ with answer \NO\ it eliminates from $E(H)$ all the pairs $\{u',v'\}$ where $u',v'\in Q_j$.

We now show that $E(H)$ contains at most $2m$ edges. Assume for the contrary that $E(H)$ contains $2m+1$ edges. Take $E'\subseteq E(H)\backslash E$ of size $m$. There is a query $Q_j$ such that ${\cal O}_G(Q_j)=\NO$ and $(\exists e\in E')\ e\subseteq Q_j$. This gives a contradiction. Now in the second round the algorithm asks a query for each edge in $E(H)$.
\end{proof}

\subsection{Lower Bound for Two-Round Deterministic Algorithm}
In this subsection we give a lower bound.
We first give two known facts from~\cite{DR82,F96}.
\begin{lem}\label{D2R1} Any deterministic nonadaptive learning algorithm for $m$-Loop must ask at least $\Omega((m^2/\log m)\log n)$ queries.
\end{lem}

\begin{lem}\label{D2R2} Let ${\cal Q}$ be a set of queries that satisfies: for every set $S\subset [n]$ of size $m$ and every $i\in S$ there is a query $Q\in {\cal Q}$ such that $S\cap Q=\{i\}$. Then $|{\cal Q}|=\Omega((m^2/\log m)\log n)$.
\end{lem}

We now prove
\begin{theorem}\label{UBDATR} Any two-round deterministic learning algorithm for $m$-Graph must ask at least $\Omega((m^2/\log m)\log n)$ queries.
\end{theorem}
\begin{proof} Suppose there is a deterministic two-round deterministic learning algorithm for $m$-Graph that asks $t=o((m^2/\log m)\log n)$ queries. Let $Q_1,\ldots,Q_t$ be the queries in the first round. Since $t=o((m^2/\log m)\log n)$, by Lemma~\ref{D2R2}, there must be a set of $\lfloor m/2\rfloor$ elements $S\subseteq [n]$ and $i\in S$ such that no query $Q_j$ satisfies $S\cap Q_j=\{i\}$. The adversary defines a set of $\lfloor m/2\rfloor$ edges $E'=\{\{i,j\}\ |\ j\in S\backslash \{i\}\}$. The other $\lceil m/2\rceil$ edges $E''=\{\{i,j\}\ |\ j\in S', S'\subseteq [n]\backslash S\}$ will be determined in the second round. The answers of the queries in the first round are determined only by the edges $E'$. That is, if $|Q_j\cap S|>1$ and $i\in Q_j$  then the answer is ``YES'', and if $i\not\in Q_j$ then the answer is \NO. After the first round no information is known about $S'$.

In the second round we need to learn $S'$ from queries $Q$ that contain $i$. Otherwise, the answer is \NO and no information is gained about $S'$. When $i\in Q$ then the problem is equivalent to learning $\lceil m/2\rceil$-Loop, and by Lemma~\ref{D2R1} we must ask at least $\Omega((m^2/\log m)\log n)$ queries.
\end{proof}

\ignore{\subsection{Lower Bound - Conjecture and intuition}
In this subsection we give a class that we believe gives a tight lower bound for two rounds and then we give an intuition to this fact.

Consider $t=m^{1/2}$ disjoint set of vertices $V_1,\ldots,V_{t}$ each contains $t^2=m^{2/3}$ vertices. Let $V'$ be the remaining vertices. For $v_1,\in V_1,\cdots,v_t\in V_t$ and sets
$U_1,\ldots,U_t\subset V'$ where $|U_i|\ge t^2$ and $|U_1|+\cdots+|U_t|=2m$ define the graph
$G(v_1,\ldots,v_t,U_1,\ldots,U_t)$ as follows: For every $i=1,\ldots,t$ define $E_i=\{\{v_i,v\}\ |\ v\in V_i\backslash \{v_i\}\}\cup \{\{v_i,u\}\ |\ u\in U_i\}$.}

\subsection{Deterministic Five-Round Algorithm}
In this section we give a deterministic five-round algorithm that asks $O(m^{2+\epsilon}\log n)$ queries.

In the first round the algorithm finds a partition of $[n]$ to $w=O(m)$ sets $S_1,\ldots, S_w$ where no edge of the target has both endpoint vertices in the same set. In the second round it learns the pairs of sets $\{S_i,S_j\}$ for which there is an edge with one endpoint in $S_i$ and the other in $S_j$. In the third and fourth rounds it finds the vertices in each set that are endpoints of some edge. Then in the fifth round it learns the edges.

For the first round we use the following result from~\cite{B15}.
\begin{lem}\label{lkj} There is a linear time algorithm that for every $n$ and $m$ constructs
a $t\times n$-matrix, $t=O(n\log m)$, $M$ with entries in $[w]$, $w=O(m)$, with the following property:
Every two columns in $M$ are equal in at most $t/(2m)$ entries.
\end{lem}
We now prove
\begin{lem} Let $M$ be the matrix in Lemma~\ref{lkj}. There is a row vector $u\in [w]^n$ in $M$ such that in the partition $\{S_{u,1},\ldots,S_{u,w}\}$ of $[n]$ where $S_{u,i}=\{j|u_j=i\}$ no edge of the target has both endpoint vertices in the same set.
\end{lem}
\begin{proof} Let $e_i=\{u_i,w_i\}$, $i=1,\ldots,m$ be the edges of the target. By Lemma~\ref{lkj},
there is at most $t/(2m)$ entries in columns $u_i$ and $w_i$ in the matrix that are equal. Therefore,
there is at least one entry (actually, at least $t/2$ entries) such that for all $i=1,\ldots,m$, columns $u_i$ and $w_i$ in the matrix are not equal in that entry. This implies the result.
\end{proof}
To know this row, for each row $u$ in $M$ and for each $j\in [w]$, the algorithm asks the query
${\cal O}_G(S_{u,j})$. Obviously, if no edge of the target has both endpoint vertices in the same set
then the answers are zeros for all $j\in [w]$. The number of queries in this round is $wt=O(m^2\log n)$.

After the first round we have a partition of $[n]$ to $w=O(m)$ sets $S_1,\ldots, S_w$ where no edge of the target has both endpoint vertices in the same set. In the second round we ask the query ${\cal O}_G(S_i\cup S_j)$ for all $i\not=j$.
If ${\cal O}_G(S_i\cup S_j)=1$ then the algorithm knows that there is an edge with one endpoint in $S_i$ and the other in $S_j$. Since $w=O(m)$, this takes $O(m^2)$ queries. If ${\cal O}_G(S_i\cup S_j)=1$ then we call $\{S_i,S_j\}$ a {\it set edge}. Obviously, there are at most $m$ set edges.

In the third and fourth rounds, the algorithm runs Cheraghchi's two-round algorithm, \cite{C13}, to find the vertices in each $S_i$ that are endpoints of some edge. This algorithm is a two-round algorithm for $m$-Loop. It asks $O(m^{1+\epsilon}\log n)$ queries for any constant $\epsilon$. We use Cheraghchi's algorithm as follows: if $\{S_i,S_j\}$ is a set edge then each query $S_i\cup (Q\cap S_j)$ for the graph is the query $Q\cap S_j$ for the $m$-Loop that contains the vertices in $S_j$ that are endpoints of the edges. Therefore, running Cheraghchi's two-round algorithm for each set edge $\{S_i,S_j\}$ gives all the vertices that are endpoints of some edge in the cut $(S_i;S_j)$. Since the number of edge sets is at most $m$ this takes $O(m^{2+\epsilon}\log m)$ queries.

In the fifth stage the algorithm exhaustively asks a query about each possible pair. That is, if for the edge set $\{S_i,S_j\}$ we have learned that $V_1\subseteq S_i$ and $V_2\subseteq S_j$ are the endpoints of some edge in the cut $(S_i;S_j)$, then we ask the queries ${\cal O}_G(\{v_1,v_2\})$ for each $v_1\in V_1$ and $v_2\in V_2$. This takes at most $m^2/2=O(m^2)$ queries.

The above algorithm implies
\begin{restatable}{theorem}{dalgo}\label{DDDDD} There is a five-round deterministic learning algorithm for $m$-Graph that asks $O(m^2\log n)$ queries.
\end{restatable}
\section{Unknown \texorpdfstring{$m$}{} - Upper Bounds}\label{six}
In this section we prove two results when $m$ is not known to the learner.

Here
$\log^{[0]}n=n$ and $\log^{[k]}n=\log\log^{[k-1]} n$. When we say w.h.p (with high probability) we mean with probability at least $1-1/\poly(n)$.
\begin{restatable}{theorem}{uthm}\label{SFR}
For every constant $k>1$ there is an $O(1)$-round Las Vegas randomized learning algorithm for $m$-Graph that asks
$O(m\log n+\sqrt{m}(\log^{[k]}n) \log n)$ queries.
\end{restatable}
\begin{restatable}{theorem}{utthm}\label{SSR}
There is a $(\log^*n)$-round Las Vegas randomized learning algorithm for $m$-Graph that asks
$O(m\log n)$ queries.
\end{restatable}
We recall Chernoff Bound
\begin{lem}{\bf (Chernoff Bound)}. Let $X_1,\ldots,X_t$ be independent random variables that takes values in $\{0,1\}$. Let $X=(X_1+\cdots +X_t)/t$ and $\mu=\E[X]$. Then for any $0<\delta<1$,
$$\Pr[|X-\mu|\ge \delta \mu]\le 2e^{-\delta^2\mu t/3},$$
and for any $\delta\ge 1$
$$\Pr[X\ge (1+\delta) \mu]\le 2e^{-\delta\mu t/3}.$$
\end{lem}

Let $Q\subseteq V$ be a $p$-random query. Let $N_G(p)$ be the probability that ${\cal O}_G(Q)=0$. It is easy to see that (see \cite{AC08})
\begin{eqnarray}\label{in01}
1-mp^2\le N_G(p)\le 1-p^2.
\end{eqnarray}

Consider $p_i$-random queries $Q_i$ for $i=1,2$. Then $Q_1\cup Q_2$ is a $(p_1+p_2-p_1p_2)$-random query. Now
\begin{eqnarray*}
N_G(p_1+p_2-p_1p_2)&=&\Pr[{\cal O}_G(Q_1\cup Q_2)=0]\\
&\le& \Pr[{\cal O}_G(Q_1)=0 \mbox{\ and\ } {\cal O}_G(Q_2)=0]\\
&=&  N_G(p_1)\cdot N_G(p_2).
\end{eqnarray*}
It is shown in \cite{AC08} that for any $G$ the function $N_G(x)$ is continuous monotonic decreasing function.
Therefore
\begin{lem}\label{PN} For any $0\le q_1< q_2\le 1$ we have $N_G(q_1)> N_G(q_2)$ and for any $0\le p_1+p_2\le 1$ we have
$$N_G(p_1+p_2)< N_G(p_1+p_2-p_1p_2)\le  N_G(p_1)\cdot N_G(p_2).$$
In particular, for any integer $k\ge 2$
$$N_G(kp)< N_G(p)^k.$$
\end{lem}

Let $p_*$ be such that $N_G(p_*)=1/2$.  By (\ref{in01}) we have
\begin{eqnarray}\label{in02}
\frac{1}{\sqrt{2m}}\le p_*\le \frac{1}{\sqrt{2}}.
\end{eqnarray}
Our first goal is to estimate $p_*$. Consider the following procedure

  \begin{center}
   \fbox{\fbox{\begin{minipage}{28em}
  \begin{tabbing}
  xxxx\=xxxx\=xxxx\=xxxx\=xxxx\= \kill
  \> \underline{{\bf Estimate}(${\cal O}_G, M$)}\\ 

  1. \>For each $p_i=1/2^i$, $i=0,1,2,\cdots$, such that $2^i\le 2^{2.5}\sqrt{M}$\\
  2. \>\> For $t=\Theta(\log n)$ independent $p_i$-queries $Q_{i,1},\ldots,Q_{i,t}$ do:\\
  3. \>\>\> $q_i=({\cal O}_G(Q_{i,1})+\cdots+{\cal O}_G(Q_{i,t}))/t $.\\
  4. \>Choose the first $p':=p_{i_0}/2$ such that $1-q_{i_0}>1/2$.\\
  5. \>If no such $i_0$ exists then output(``$m>M$'').\\
  6. \>Otherwise output($p'$)  $\backslash\star\  p_*\ge p'\ge p_*/8\  \star\backslash$.
  \end{tabbing}\end{minipage}}}
  \end{center}

We now show
\begin{lem}\label{lmln} Let $M\in [n]$ be any integer. The procedure {\bf Estimate}(${\cal O}_G, M$), w.h.p,
asks $\Theta((\log M)(\log n))$ queries
and either outputs $p'$ such that
$p_*\ge p'\ge p_*/8$ or proclaims that $m>M$.
\end{lem}
\begin{proof}
Let $k$ be such that $p_*/2<p_k\le p_*$. Then $p_{k-2}>2p_*$ and therefore, by Lemma~\ref{PN}, for all $i\le k-2$, $N_G(p_{i})\le N_G(p_{k-2})<N_G(2p_*)<N_G(p_*)^2\le 1/4$. Therefore, by Chernoff bound ($X_j=1-{\cal O}_G(Q_{i,j})$, $t=\Theta(\log n)$, $\mu=N_G(p_i)\le 1/4$ and $\delta\mu=1/4$)
$$\Pr[i_0\le k-2]\le \Pr[(\exists i\le k-2)\ 1-q_{i}>1/2]\le 1/\poly(n).$$
Therefore, w.h.p $i_0\ge k-1$.
Also $p_{k+1}=p_k/2\le p_*/2$ and therefore, by Lemma~\ref{PN}, $N_G(p_{k+1})>N_G(p_*/2)>N_G(p_*)^{1/2}>0.7$. Therefore, by Chernoff bound, ($X_j=1-{\cal O}_G(Q_{k+1,j})$, $t=\Theta(\log n)$, $\mu=N_G(p_{k+1})\ge 0.7$ and $\delta\mu=0.2$)
$$\Pr[i_0\ge k+2]\le \Pr[1-q_{k+1}<1/2]\le 1/\poly(n).$$
Therefore, w.h.p $k-1\le i_0\le k+1$ and then
$2p_*\ge p_{k-1}\ge p_{i_0}\ge p_{k+1}\ge p_*/4$. Therefore w.h.p
$$p_*\ge p'=p_{i_0}/2\ge p_*/8.$$

Let $i'$ be such that $2^{i'}\le 2^{2.5}\sqrt{M}<2^{i'+1}$.
If $i_0$ does not exist then w.h.p $i'\le k$ and therefore by (\ref{in02}),
$$2^{2.5}\sqrt{M}<2^{i'+1}\le 2^{k+1}=\frac{1}{p_{k+1}}\le \frac{4}{p_*}=4\sqrt{2m}$$
and then $m>M$.
\end{proof}

The following procedure estimates $p_*$ in $k$ rounds
  \begin{center}
   \fbox{\fbox{\begin{minipage}{28em}
  \begin{tabbing}
  xxxx\=xxxx\=xxxx\=xxxx\=xxxx\= \kill
  \> \underline{{\bf $k$-Estimate}(${\cal O}_G$)}\\ 

  1. \>For $j=k-1$ downto $0$\\
  2. \>\> $M_j\gets (\log^{[j]}n)^2$.\\
  3. \>\>{\bf Estimate}(${\cal O}_G,M_j$)\\
  4. \>\> If the output is $p'$ then Goto 6.\\
  5. \> EndFor.\\
  6. \>Output($p'$)  $\backslash\star\  p_*\ge p'\ge p_*/8\  \star\backslash$.\\
  \end{tabbing}\end{minipage}}}
  \end{center}
We now show
\begin{lem} The procedure {\bf $k$-Estimate}(${\cal O}_G$) runs in $k$ rounds and w.h.p asks $$O((\log^{[k]}n)(\log n)+m\log n)$$ queries
and outputs $p'$ such that
$p_*\ge p'\ge p_*/8$.
\end{lem}
\begin{proof} For $j=0$ we have $M_0=(\log^{[0]}n)^2=n^2$ and since $m<M_0$ the algorithm must stops and output such $p'$. It remains to prove the query complexity.

If {\bf Estimate}(${\cal O}_G,M_{k-1})$ returns $p'$ then, by Lemma~\ref{lmln}, the algorithm w.h.p asks $$\Theta((\log M_{k-1})(\log n))=\Theta((\log^{[k]}n)(\log n))$$ queries. Otherwise, $m>M_{k-1}$.
Let $j$ be such that $M_{j+1}< m \le M_j$. Then $p'$ is returned by some {\bf Estimate}(${\cal O}_G,M_{j'})$ where $j'\ge j$. Therefore, by Lemma~\ref{lmln}, w.h.p the number of queries is $O$ of
$$\sum_{i=j}^k (\log M_i)(\log n)\le 2 (\log M_j)(\log n)=4\sqrt{M_{j+1}}\log n\le 4m\log n.$$
\end{proof}
In particular,
\begin{corollary} The procedure {\bf $\log^*n$-Estimate}(${\cal O}_G$) runs in $\log^*n$ rounds, w.h.p asks $O(m\log n)$ queries
and outputs $p'$ such that
$p_*\ge p'\ge p_*/8$.
\end{corollary}

The following is Lemma~{4.1} in~\cite{AC08}
\begin{lem} \label{4.1} Suppose $I$ is an independent set in $G$ and let $\Gamma(I)$ be the set of neighbors of the vertices in $I$. For a $p$-random query $Q$ we have
$$\Pr[{\cal O}_G(Q)=0\ |\ I\subseteq Q]\ge (1-p)^{|\Gamma(I)|}\cdot N_G(p)\ge (1-p)^{|\Gamma(I)|}\cdot (1-mp^2).$$
\end{lem}
The following is Lemma~{5.2} in~\cite{AC08}. We give the proof for completeness
\begin{lem}~\label{t830} Let $p_*\ge p'\ge p_*/8$. For any constant $c'$ there is a constant $c$ such that:
If $\{u,v\}\not\in E(G)$ and $\deg_G(u)+\deg_G(v)\le c'/p'$ then for $p'$-random query~$Q$
$$\Pr[{\cal O}_G(Q)=0 \ \mbox{and\ } \{u,v\}\subseteq Q]\ge {c}{p'^2}.$$
\end{lem}
\begin{proof} First notice that since $p'\le p_*\le 1/\sqrt{2}$ we have $(1-p')^{1/p'}\ge (1-p_*)^{1/p_*}\ge 0.17$.
By Lemma~\ref{4.1}, the  probability that ${\cal O}_G(Q)=0 \ \mbox{and\ } \{u,v\}\subseteq Q$ is equal to
\begin{eqnarray*}
\Pr[\{u,v\}\subseteq Q]\cdot \Pr[{\cal O}_G(Q)=0 \ |\ \{u,v\}\subseteq Q]&\ge& p'^2 (1-p')^{c'/p'} N_G(p')\\
&\ge& p'^2(0.17)^{c'} N_G(p_*)={c}p'^2.
\end{eqnarray*}
\end{proof}
We use Lemma~\ref{t830} for the following

  \begin{center}
   \fbox{\fbox{\begin{minipage}{28em}
  \begin{tabbing}
  xxxx\=xxxx\=xxxx\=xxxx\=xxxx\= \kill
  \> \underline{{\bf Split}(${\cal O}_G$)}\\ 

  1. \>$E(H)\gets \{\{u,v\}\ |\ u,v\in V, u\not= v\}$.\\
  2. \>Choose $t=\Theta((1/p')^2\log n)$ $p'$-random queries $Q_1,\ldots,Q_t$\\
  3. \>For each query $Q_i$.\\
  4. \>\>If ${\cal O}_G(Q_i)=0$\ then\\
  5. \>\> \>For every $u,v\in Q_i$ do $E(H)\gets E(H)\backslash \{\{u,v\}\}$\\
  6. \>EndFor.\\
  7. \>$V_1(H)\gets \{u\in V\ |\ \deg_H(u)\ge 3/p'\}$, $V_2(H)=V\backslash V_1$
  \end{tabbing}\end{minipage}}}
  \end{center}
\begin{lem}\label{split} Let $p_*\ge p'\ge p_*/8$.
The procedure {\bf Split}(${\cal O}_G$) asks at most $O(m\log n)$ queries and w.h.p the following hold:
\begin{enumerate}
\item For all $u,v\in V_2(H)$, $\{u,v\}\in E(H)$ if and only if $\{u,v\}\in E(G)$.

In particular,

\item If $\deg_G(u)\le 1/p'$ and $\{u,v\}\in E(H)\backslash E(G)$ then $\deg_G(v)>1/p'$.
\end{enumerate}
\end{lem}
\begin{proof} By (\ref{in02}), the number of queries is $O((1/p'^2)\log n)=O((1/p_*^2)\log n)=O(m\log n)$.

If ${\cal O}_G(Q_i)=0$ and $u,v\in Q_i$ then $\{u,v\}$ is not an edge in $G$.
Therefore, procedure {\bf Split}(${\cal O}_G$) only removes edges in $E(H)$ that are not in $E(G)$. Therefore, if $\{u,v\}\in E(G)$ then $\{u,v\}\in E(H)$.
Suppose $u,v\in V_2(H)$ and $\{u,v\}\not\in E(G)$. Since $\deg_H(u),\deg_H(v)< 3/p'$ we also have $\deg_G(u),\deg_G(v)< 3/p'$.
Therefore, by Lemma~\ref{t830},
\begin{eqnarray*}
\Pr[\{u,v\}\in E(H)]&=&\Pr[(\forall i)\ {\cal O}_G(Q_i)=1 \ \mbox{or\ } \{u,v\}\not\subseteq Q]\\
&\le& \left(1- cp'^2\right)^t=\frac{1}{\poly(n)}.
\end{eqnarray*}
\end{proof}

The following is Lemma~{5.4} in~\cite{AC08}. We give the proof for completeness
\begin{lem}\label{nlpp} Let $p_*\ge p'\ge p_*/8$. There are at most $2/p'\le 16/p_*\le 16\sqrt{2m}$ vertices in $G$ that have degree more than $1/p'$.
\end{lem}
\begin{proof} Suppose $h$ vertices, $v_1,\ldots,v_h$, have degree more than $1/p'$.
Let $Q$ be a $p'$-random query. Then
\begin{eqnarray*}
\frac{1}{2}&=&N_G(p_*)\le N_G(p')=\Pr[{\cal O}_G(Q)=0]\\
&=& \Pr[{\cal O}_G(Q)=0\ |\ (\exists i) v_i\in Q]\cdot \Pr[(\exists i) v_i\in Q]\\
&&+ \Pr[{\cal O}_G(Q)=0\ |\ (\forall i) v_i\not\in Q]\cdot \Pr[(\forall i) v_i\not\in Q]\\
&\le& (1-p')^{1/p'} (1-(1-p')^h)+(1-p')^h\\
&\le& e^{-1}+(1-e^{-1})(1-p')^h\le e^{-1}+(1-e^{-1})e^{-p'h}.
\end{eqnarray*}
Therefore $h\le 2/p'$.
\end{proof}

\begin{lem}\label{voneh} Let $V_3=\{v\in G\ |\ \deg_G(v)\ge 1/p'\}$. Then w.h.p $V_1(H)\subseteq V_3$.

In particular, w.h.p $|V_1(H)|\le |V_3|\le 1/p'\le 8\sqrt{2m}.$
\end{lem}
\begin{proof} CHECK AGAIN If $v\not\in V_3$ then $\deg_G(v)<1/p'$. Now if $\{u,v\}$ is an edge in $H$
but not in $G$ then by Lemma~\ref{split}, $\deg(u)\ge 1/p'$. Since by Lemma~\ref{nlpp} the number
of vertices of degree more than $1/p'$ is less than $2/p'$ we have $\deg_H(v)<\deg_G(v)+2/p'\le 3/p'$.
Therefore $v\not\in V_1(H)$.
\end{proof}

Therefore, it remains to learn the edges of the vertices in $V_1(H)$. Since $V_1(H)\subseteq V_3$ we have
that
\begin{lem}
For every $v\in V_1(H)$, $d_v:=\deg_G(v)\ge 1/p'$.
\end{lem}

This is one of the main properties that we will use in the sequel.

Let $Q$ be a $p$-random query and $u\in V$ be a vertex of degree $d_u\ge 1/p'$. Let $N_{u,G}(p)$ be the probability that ${\cal O}_G(Q\cup \{u\})=0$. As before, $N_{u,G}(p)$ is monotonically decreasing and $N_{u,G}(kp)\le N_{u,G}(p)^k$.
Let $p_u$ be the probability such that
\begin{eqnarray}\label{eee}
N_{u,G}(p_u)= e^{-1}.
\end{eqnarray}

Next, we show that estimating $p_u$ implies estimating $d_u$.
Since
$$e^{-1}=N_{u,G}(p_u)\le (1-p_u)^{d_u} \le e^{-p_ud_u}$$ we have $p_u\le 1/d_u\le p'$.
Notice that $N_G(p_u)\ge N_G(p')\ge N_G(p_*)=1/2$. Now by Lemma~\ref{4.1},
$$e^{-1}=N_{u,G}(p_u)\ge (1-p_u)^{d_u} N_G(p_u)\ge \frac{1}{2}(1-{d_u}p_u).$$  Therefore
\begin{eqnarray}\label{kkk}
1-\frac{2}{e}\le p_ud_u\le 1.
\end{eqnarray}

\begin{lem}\label{EDM} Let $M\in [0,n]$.
The procedure {\bf EstimateDegree}(${\cal O}_{G},u, M$) asks $\Theta((\log M)(\log n))$ queries
and either output $p'_u\ge 1/(2M)$ such that
$p_u\ge p'_u\ge p_u/8$ and then $d_u\le 32M$ or proclaims that $d_u>M$.
\end{lem}
  \begin{center}
   \fbox{\fbox{\begin{minipage}{28em}
  \begin{tabbing}
  xxxx\=xxxx\=xxxx\=xxxx\=xxxx\= \kill
  \>\underline{{\bf EstimateDegree(${\cal O}_{G},u, M$)}}\\ 

  1. \>For each $p_{u,i}=1/2^i$ such that $2^i\le 16M$\\
  2. \>\> for $t=\Theta(\log n)$ independent $p_{u,i}$-queries $Q_{i,1},\ldots,Q_{i,t}$ do:\\
  3. \>\>\> $q_{u,i}=({\cal O}_G(Q_{i,1}\cup\{u\})+\cdots+{\cal O}_G(Q_{i,t}\cup\{u\}))/t $.\\
  4. \>Choose the first $p'_u:=p_{i_{u,0}}/2$ such that $1-q_{u,i_{u,0}}>1/e$.\\
  5. \>If no such $i_{u,0}$ exists then output(``$d_u>M$'').\\
  6. \>Otherwise output($p_u'$)  $\backslash\star\  p_u\ge p_u'\ge p_u/8\  \star\backslash$.
  \end{tabbing}\end{minipage}}}
  \end{center}

\begin{proof}
Let $k$ be such that $p_u/2<p_{u,k}\le p_u$. Then $p_{u,k-2}>2p_u$ and therefore for all $i\le k-2$, $N_{u,G}(p_{u,i})\le N_{u,G}(p_{u,k-2})<N_{u,G}(2p_u)<N_{u,G}(p_u)^2\le 1/e^2$. Therefore, ($\E[1-q_{u,i}]=N_{u,G}(p_{u,i})\le  1/e^2$)
$$\Pr[i_{u,0}\le k-2]\le \Pr[(\exists i\le k-2)1-q_{u,i}>1/e]\le 1/\poly(n).$$
Therefore, w.h.p, $i_{u,0}\ge k-1$.
Now, $p_{u,k+1}=p_{u,k}/2\le p_u/2$ and therefore $N_{u,G}(p_{u,k+1})>N_{u,G}(p_u/2)>N_{u,G}(p_u)^{1/2}>0.6$. Therefore
$$\Pr[i_{u,0}\ge k+2]\le \Pr[1-q_{u,k+1}<1/e]\le 1/\poly(n).$$
Therefore, w.h.p $k-1\le i_{u,0}\le k+1$ and then
$2p_u\ge p_{u,k-1}\ge p_{u,i_{u,0}}\ge p_{u,k+1}\ge p_u/4$. Therefore w.h.p
$$p_u\ge p'_u=p_{u,i_{u,0}}/2\ge p_u/8.$$
Now, by (\ref{kkk}) and step 1 in the procedure,
$$d_u\le \frac{1}{p_u}\le \frac{1}{p_u'}=\frac{2}{p_{u,i_{u,0}}}\le 32M.$$

Let $i'$ be such that $2^{i'}\le 16M$ and $2^{i'+1}>{16M}$.
If no such $i_{u,0}$ exists then $i'\le k$ and therefore by (\ref{kkk}),
$$16M<2^{i'+1}\le 2^{k+1}=\frac{1}{p_{u,k+1}}\le \frac{4}{p_u}\le\frac{4 d_u}{(1-2/e)}$$
and then $m\ge d_u> M$.
\end{proof}

  \begin{center}
   \fbox{\fbox{\begin{minipage}{28em}
  \begin{tabbing}
  xxxx\=xxxx\=xxxx\=xxxx\=xxxx\= \kill
  \> \underline{{\bf $k$-EstimateDegree}(${\cal O}_G,u$)}\\ 

  1. \>For $j=k-1$ downto $0$\\
  2. \>\> $M_j\gets (\log^{[j]}n)^2$.\\
  3. \>\>{\bf EstimateDegree}(${\cal O}_G,u,M_j$)\\
  4. \>\> If the output is $p_u'$ then Goto 6.\\
  5. \> EndFor.\\
  6. \>Output($p_u'$)  $\backslash\star\  p_u\ge p_u'\ge p_u/8\  \star\backslash$.\\
  \end{tabbing}\end{minipage}}}
  \end{center}

We now show
\begin{lem}\label{ED02} The procedure {\bf $k$-EstimateDegree}(${\cal O}_G,u$) runs in $k$ rounds, with probability $1-1/\poly(n)$, asks $$O((\log^{[k]}n)(\log n)+\sqrt{m}\log n)$$ queries
and outputs $p_u'$ such that
$p_u\ge p_u'\ge p_u/8$.
\end{lem}
\begin{proof} For $j=0$ we have $M_0=(\log^{[0]}n)^2=n^2$ and since $m<M_0$, by Lemma~\ref{EDM}, the algorithm must stop and output such $p_u'$. It remains to prove the query complexity.

If {\bf EstimateDegree}(${\cal O}_G,u,M_{k-1})$ returns $p_u'$ then, by Lemma~\ref{EDM}, the algorithm asks $$\Theta((\log M_{k-1})(\log n))=\Theta((\log^{[k]}n)(\log n))$$ queries. Otherwise, $m\ge d_u>M_{k-1}$.
Let $j$ be such that $M_{j+1}< m \le M_j$. Then $p_u'$ is returned by some {\bf EstimateDegree}(${\cal O}_G,u,M_{j'})$ where $j'\ge j$. Therefore the number of queries is $O$ of
$$\sum_{i=j'}^k (\log M_i)(\log n)\le 2 (\log M_j)(\log n)=4\sqrt{M_{j+1}}\log n\le 4\sqrt{m}\log n.$$
\end{proof}
In particular,
\begin{corollary} The procedure {\bf $\log^*n$-EstimateDegree}(${\cal O}_G$) runs in $\log^*n$ rounds, with probability $1-1/\poly(n)$, asks $O(\sqrt{m}\log n)$ queries
and outputs $p_u'$ such that
$p_u\ge p_u'\ge p_u/8$.
\end{corollary}

In particular
\begin{lem} The degrees of all the vertices $u$ in $V_1(H)$ can be estimated with probability $1-1/\poly(n)$
by $1/p_u'$ that falls in the range $[d_u,31d_u]$ with
\begin{enumerate}
\item $k$-round algorithm that asks at most $O((\log^{[k]}n)\sqrt{m}\log n+m\log n)$ queries.
\item $\log^*n$-round algorithm that asks at most $O(m\log n)$ queries.
\end{enumerate}
\end{lem}
\begin{proof}
By Lemma~\ref{voneh}, w.h.p $|V_1(H)|=O(\sqrt{m})$ and by (\ref{kkk}) and Lemma~\ref{ED02}, for every $u\in V_1(H)$, the value of $1/p_u'$ is bounded by
$$d_u\le \frac{1}{p_u}\le \frac{1}{p_u'}\le \frac{8}{p_u}\le \frac{8d_u}{1-2/e}\le 31d_u.$$
\end{proof}

After estimating the degree the algorithm finds the edges of each $v\in V_1(H)$ in $G$.

We first have (the proof is the same as of Lemma~\ref{4.1})
\begin{lem} \label{4.11} Suppose $I$ is an independent set in $G$ and let $\Gamma(I)$ be the set of neighbors of the vertices in $I$.
Let $u$ be a vertex in $G$ such that $u\not\in\Gamma(I)$. For a $p$-random query $Q$ we have
$$\Pr[{\cal O}_G(Q\cup \{u\})=0\ |\ I\subseteq Q]\ge (1-p)^{|\Gamma(I)|}\cdot N_{u,G}(p).$$
\end{lem}
We now prove

\begin{lem}
The edges of each $v\in V_1(H)$ can be learned with probability $1-1/\poly(n)$ in $O(m\log n)$ queries
\end{lem}
\begin{proof}
Let $u\in V_1(H)$. Let $Q$ be a $p_u'$-random query and $\{u,v\}\not\in E$. By Lemma~\ref{4.11}, Lemma~\ref{ED02}, (\ref{eee})
\begin{eqnarray*}
\Pr[{\cal O}_G(Q\cup \{u\})=0\mbox{\ and \ } v\in Q]&=&\Pr[v\in Q]\cdot\Pr[{\cal O}_G(Q\cup \{u\})=0\ |\ v\in Q]\\
&=& p_u'\cdot (1-p_u')^{1/p_u'}N_{u,G}(p_u')\\
&\ge & \frac{p_u}{8} (1-p_u')^{1/p_u'}N_{u,G}(p_u)\\
&\ge& \frac{1-2/e}{8d_u} \frac{1}{4}e^{-1}\\
&\ge& \frac{0.003}{d_u}.
\end{eqnarray*}
Therefore each $p_u'$-random query discover that $\{u,v\}$ is not an edge with probability at least $0.003/d_u$.
Therefore $O(d_u\log n)$ queries are enough to find all the neighbours of $u$ with probability $1-1/\poly(n)$.
The total number of queries is $O$ of
$$\sum_{u\in V_1(H)} d_u\log n\le 2m\log n.$$
\end{proof}
The following is the procedure
  \begin{center}
   \fbox{\fbox{\begin{minipage}{28em}
  \begin{tabbing}
  xxxx\=xxxx\=xxxx\=xxxx\=xxxx\= \kill
  \> \underline{{\bf FindEdges}(${\cal O}_G$,$V_1(H)$)}\\ 

  1. \>For each $u\in V_1(H)$ do.\\
  2. \>\>Choose $t=\Theta((1/p_u')\log n)$ $p_u'$-random queries $Q_1,\ldots,Q_t$\\
  3. \>\>For each query $Q_i$.\\
  4. \>\>\>If ${\cal O}_G(Q_i\cup \{u\})=0$ then\\
  5. \>\>\> For every $v\in Q_i$ do $E(H)\backslash \{\{u,v\}\}$\\
  6. \>EndFor\\
  7. \> Return $(V,E(H))$.
  \end{tabbing}\end{minipage}}}
  \end{center}

\section{Open Problems}
In this section we give some open problems
\begin{enumerate}
\item The problem of whether there is a $O(1)$-round learning algorithm for $m$-Graph with $O(m\log n)$ queries when $m$ is unknown to the learner is still open.
\item Find a tight lower bound for two-round Monte-Carlo algorithm.
\end{enumerate}

\bibliography{bibfile.bib}

\appendix

\begin{figure}
\centering
\includegraphics[trim = 2cm 0cm 0cm 0cm,width=1.5\textwidth]{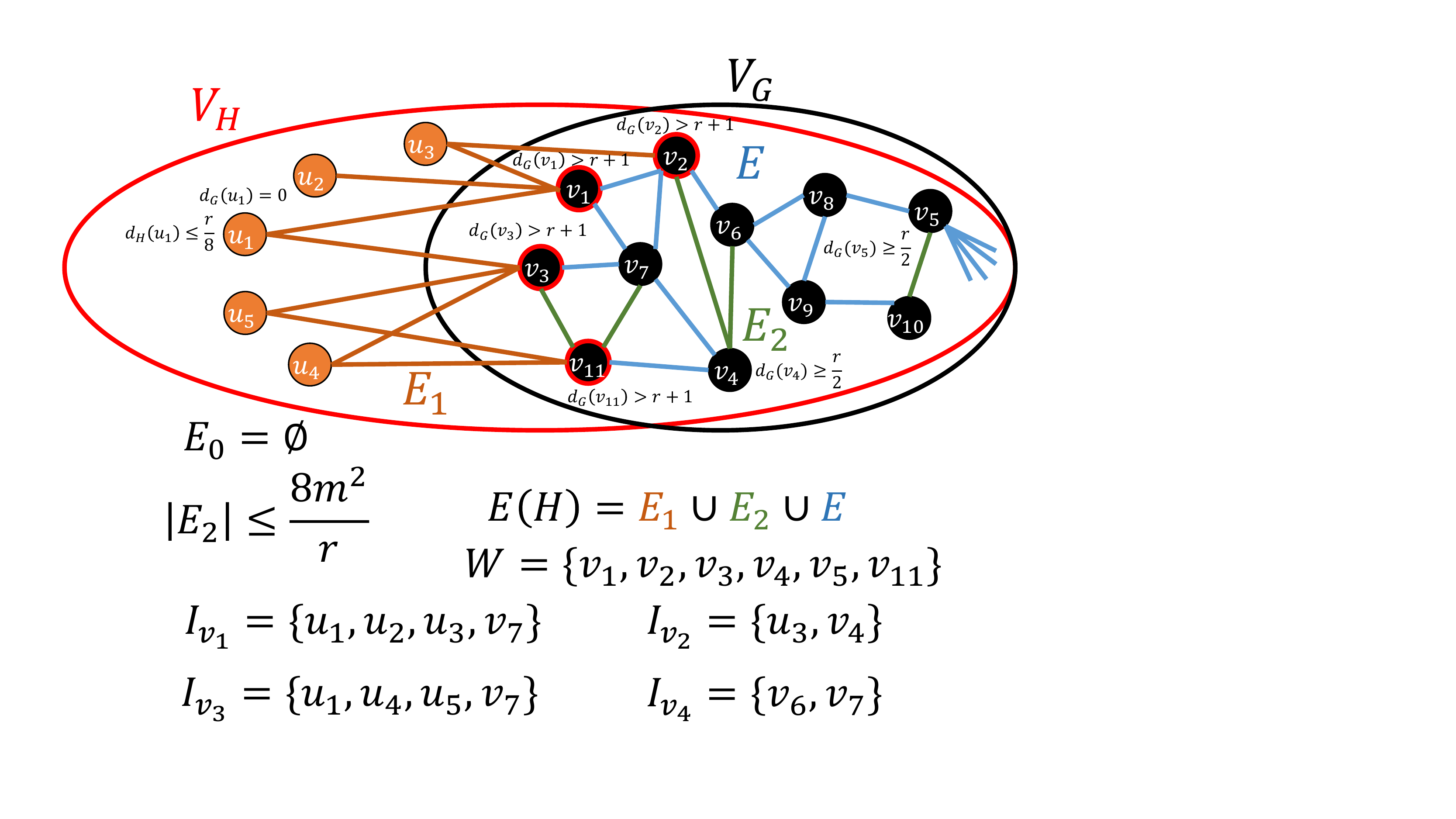}
\caption{An Example}
\label{Proof1}
\end{figure}






\end{document}